\documentclass{article}

\PassOptionsToPackage{numbers}{natbib}

\usepackage[final]{neurips_2024}

\usepackage[utf8]{inputenc} % allow utf-8 input
\usepackage[T1]{fontenc} % use 8-bit T1 fonts
\usepackage[pdftex,colorlinks,linkcolor=teal,citecolor=teal,filecolor=teal,urlcolor=teal]{hyperref}\usepackage{url} % simple URL typesetting
\usepackage{booktabs} % professional-quality tables
\usepackage{amsfonts} % blackboard math symbols
\usepackage{nicefrac} % compact symbols for 1/2, etc.
\usepackage{microtype} % microtypography
\usepackage{xcolor} % colors

\usepackage{tikz}
\usepackage{graphicx}
\usepackage{nicefrac}
\usepackage{aliascnt}
\usepackage{amssymb,amsmath,amsthm,sectsty,url}
\usepackage{cleveref}
\usepackage{color}
\usepackage[ruled,vlined]{algorithm2e}
\usepackage{booktabs}
\usepackage{multirow}
\usepackage{makecell}
\usepackage{multicol}
\usepackage{subcaption}

\usepackage{float}

\SetKwInput{KwInput}{input}
\SetKwInput{KwOutput}{Output}
\SetKwInput{KwReturn}{return}

\newtheorem{theorem}{Theorem}[section]
\crefname{theorem}{Theorem}{Theorems}

\newaliascnt{lemma}{theorem}
\newtheorem{lemma}[lemma]{Lemma}
\aliascntresetthe{lemma}
\crefname{lemma}{Lemma}{Lemmas}

\newaliascnt{proposition}{theorem}

\aliascntresetthe{proposition}
\crefname{proposition}{Proposition}{Propositions}

\newaliascnt{corollary}{theorem}

\aliascntresetthe{corollary}
\crefname{corollary}{Corollary}{Corollaries}

\newaliascnt{fact}{theorem}

\aliascntresetthe{fact}
\crefname{fact}{Fact}{Facts}

\newaliascnt{definition}{theorem}

\aliascntresetthe{definition}
\crefname{definition}{Definition}{Definitions}

\newaliascnt{remark}{theorem}

\aliascntresetthe{remark}
\crefname{remark}{Remark}{Remarks}

\newaliascnt{conjecture}{theorem}

\aliascntresetthe{conjecture}
\crefname{conjecture}{Conjecture}{Conjectures}

\newaliascnt{claim}{theorem}

\aliascntresetthe{claim}
\crefname{claim}{Claim}{Claims}

\newaliascnt{question}{theorem}

\aliascntresetthe{question}
\crefname{question}{Question}{Questions}

\newaliascnt{exercise}{theorem}

\aliascntresetthe{exercise}
\crefname{exercise}{Exercise}{Exercises}

\newaliascnt{example}{theorem}

\aliascntresetthe{example}
\crefname{example}{Example}{Examples}

\newaliascnt{notation}{theorem}

\aliascntresetthe{notation}
\crefname{notation}{Notation}{Notations}

\newaliascnt{problem}{theorem}

\aliascntresetthe{problem}
\crefname{problem}{Problem}{Problems}

\newcommand{\norm}[1]{\lVert#1\rVert}

\def\E{\mathbb E}
\newcommand{\R}{\mathbb R}

\definecolor{methodred}{RGB}{176,36,24}
\definecolor{methodpurple}{RGB}{104,52,154}
\definecolor{methodblue}{RGB}{79,113,190}
\definecolor{methodgreen}{RGB}{94,169,63}

\definecolor{goodgreen}{RGB}{74,146,59}
\definecolor{badred}{RGB}{196,46,34}

\usepackage{makecell}
\usepackage{chngpage}

\newcommand{\alg}{DP-KPS}

\title{Private Text Generation by Seeding Large Language Model Prompts}

\author{%
 Supriya Nagesh\thanks{Authors contributed equally. Correspondence to: \texttt{nsupriy@amazon.com} and \texttt{justc@mit.edu}}\ \ \thanks{Amazon}
  \and
  Justin Y.~Chen\footnotemark[1]\ \ \thanks{MIT. Work done while an intern in Amazon}
  \and
  Nina Mishra\footnotemark[2]
  \and
  Tal Wagner\thanks{Tel-Aviv University and Amazon}
}

\begin{document}
\maketitle

\begin{abstract}
We explore how private synthetic text can be generated by suitably prompting a large language model (LLM).
This addresses a challenge for organizations like hospitals, which hold sensitive text data like patient medical records, and wish to share it in order to train machine learning models for medical tasks, while preserving patient privacy.  
Methods that rely on training or finetuning a model may be out of reach, either due to API limits of third-party LLMs, or due to ethical and legal prohibitions on sharing the private data with the LLM itself. 

We propose Differentially Private Keyphrase Prompt Seeding (\alg), a method that generates a private synthetic text corpus from a sensitive input corpus, by accessing an LLM only through privatized prompts. 
It is based on seeding the prompts with private samples from a distribution over phrase embeddings, thus capturing the input corpus while achieving requisite output diversity and maintaining differential privacy.
We evaluate \alg\ on downstream ML text classification tasks, and show that the corpora it generates preserve much of the predictive power of the original ones.   
Our findings offer hope that institutions can reap ML insights by privately sharing data with simple prompts and little compute. 
\end{abstract}

%This paper explores how private synthetic text can be generated by suitably prompting a large language model (LLM).
%The motivation comes from organizations who hold private data and wish to share their data while preserving privacy.  
%Prior methods focus on privately fine-tuning an LLM. 
%However, such a solution may be out of reach for organizations that lack access to the massive number of learned LLM model weights and/or compute resources required to fine-tune an LLM.  Keyword-based prompts are investigated, specifically those drawn from a private kernel density distribution over keyword embeddings.  New methods are given for sampling a sequence of keywords privately.   To evaluate private prompts, experiments compare their effectiveness on downstream ML tasks.  Privately-prompted, LLM-generated data are shown to have similar predictive power as the original data.   Our findings offer hope that institutions can privately share data with simple prompts, little compute and yet still reap ML insights.

\section{Introduction}
\label{sec:intro}

Organizations have large text corpora that they wish to share with others for analytical insights.  When this data contains information about people, privacy considerations abound.  For example, a hospital that collects clinical notes may wish to predict who is at risk of developing a particular health condition.  
They wish to share their clinical notes so others can train a model.  However, privacy and legal regulations prevent data sharing in the clear.

\noindent\textbf{Motivating case.}
We are motivated by the following true story. A hospital has an unexpected rise in deaths due to a particular condition and hence wants to predict who is at risk so preventative measures can be taken. An external organization is consulted to build an ML model, since no in-hospital ML expertise is available. Since patient records are private and hospitals do not know how to make records private, one year passes, data is not shared, and the surge in deaths remains unaddressed.

Thus, we are driven to find ways to generate useful yet privacy-preserving synthetic records. 'Useful', in this context, means the generated corpus of synthetic records could serve as an alternative training set for downstream ML models that would normally be trained on real textual records. 

%\noindent\textbf{Risks of plain PII removal.}
How should `privacy' be defined?
One common practice %to overcome privacy constraints 
is anonymization, where personally identifiable information (PII) such as name, address, social security number is deliberately concealed.  Tools for anonymizing text such as~\cite{philter,comprehendMedical} can identify PII to anonymize with high accuracy.  However, even if PII Is obscured, the rest of the text may still be stitched together to identify an individual by a unique set of traits (e.g., a set of medical conditions).  Moreover, anonymized data may be joined with other public datasets to re-identify individuals, such as car accident data~\cite{car_accidents} or voting records~\cite{latanya}.  
This led researchers to look for more rigorous notions of privacy that provide formal guarantees. 

\noindent\textbf{Differential Privacy} (DP)
\cite{dwork2006calibrating} has emerged as a rigorous alternative, and is now considered the gold standard for privacy in data analysis and machine learning. Intuitively, it ensures that the result of a computation over a dataset is not significantly altered by the inclusion of any individual in the dataset, and thus information about any individual cannot be recovered from the result. 

DP can be enforced at different parts of the data analysis pipeline. Perhaps the most challenging one is \emph{synthetic data generation} \cite{ponomareva2023dp}. 
The goal is to create a synthetic dataset, which is DP w.r.t.~the original private dataset, yet preserves its essential global properties for purposes of analysis and machine learning. Due to the post-processing property of DP, the synthetic dataset can be released and used as a proxy in downstream computations without further privacy concerns. 
In this work, our goal is to create synthetic private collections of text documents.

\noindent\textbf{Large Language Models} (LLMs)
are by now renowned for their ability to generate high-quality natural language text. 
Thus, it is natural to seek to harness their power for creating private synthetic text \cite{hu2023sok}, and much work has been dedicated to training them privately (cf.~\Cref{sec:related}). 
However, in many real scenarios, these methods are difficult to apply. Due to the extreme cost and complexity of training and maintaining LLMs, most organizations use proprietary LLM services offered by third parties, who often limit the ability to train or modify the LLM. Moreover, the LLM, being owned by a third party, may itself be considered a privacy risk, and client organizations are often unable to share their sensitive data with it. 
In fact, even publicly available medical datasets widely used in research, such as MIMIC \cite{goldberger2000physiobank}, enforce limits on using their data in LLM prompts in the clear.\footnote{\url{https://physionet.org/news/post/gpt-responsible-use}}

As a result, several works have recently emphasized the necessity and importance of developing machine learning methods that interact with large pre-trained models as \emph{interfaces}, accessing them only through inference or prompt, and feeding them only with data which is safe for public release \cite{pryzant-etal-2023-automatic,duan2023flocks,lin2023differentially,cohen2023hot}. 
Such methods enable a wider audience of prospective users---including those with limited access and resources---to harness state-of-the-art machine learning for their ends, alleviating many logistical and regulatory barriers. 
Our work subscribes to this emerging trend. 

\noindent\textbf{Problem specification.}
The foregoing discussion leads us to the following challenge: 
Given a private dataset $\mathcal D$ of text documents, create a privatized synthetic dataset of text documents, which can be used as a proxy for $\mathcal D$ in downstream machine learning. At our disposal is an LLM, but it can be accessed only through a prompting interface, and is only allowed to see data which is already privatized. Privacy is defined as document-level DP, where each text document in $\mathcal D$ corresponds to an individual (e.g., a medical record), and thus the output of our method is required to be insensitive to the inclusion of any single entire document in $\mathcal D$. See \Cref{sec:prelim} for the formal definition.

%\textcolor{red}{Cost is a consideration since text is generated via an LLM, and LLMs charge by the number of input and output tokens. 
%Combining the above, the problem is to generate synthetic, privatized, small yet predictive, text corpora at a low-cost. }

\subsection{Our Method: \alg}\label{sec:intro_method}
We outline our method, Differentially Private  Keyphrase Prompt Seeding. See also \Cref{fig:method2,fig:flow}.

Our starting point is the recent work \cite{eldan2023tinystories}. 
Unrelated to privacy, their goal was to generate a synthetic dataset of children's stories with an LLM.  
A major challenge turned out to be \emph{output diversity}:
they found that when prompting the LLM for stories without further specifics, even at a high setting of the temperature parameter, the resulting stories were not diverse enough to accomplish the downstream task (which in their case was to train a smaller language model). Their solution was to prompt the LLM to write a story that contains three specific words (e.g., the adjective `bad', the noun `wolf' and the verb `huff'). By prompting each time with different words, drawn at random from a pre-specified vocabulary, they achieved the requisite level of diversity in the synthetic dataset of stories. 

We wish to adapt this idea, of \emph{seeding the LLM prompt with keyphrases},\footnote{We opt to use \emph{keyphrase} instead of \emph{keyword}, since in many application contexts, vocabulary elements consist of several words (e.g., ``congestive heart failure'' in medical records).} to our problem of generating DP text. 
While \cite{eldan2023tinystories} could choose their keyphrases at random (having no privacy constraints, nor a given dataset to begin with), our challenge is to choose them in a way that captures the given private dataset $\mathcal D$, while adhering to DP.  The general approach is illustrated in Figure~\ref{fig:method2}.

% \begin{figure*}[ht]
% \vskip 0.2in
% \begin{center}
% \centerline{\includegraphics[width=\textwidth]{figures/fig_idea}}
% \caption{ General idea behind the approach. We start with a collection of private documents.  Next, a private sequence of keyphrases is generated (the key technical contribution). The keyphrases are then used to prompt an LLM for a document containing the keyphrases.  This process is repeated multiple times to obtain a collection of private synthetic documents.}
% \label{fig:general_idea}
% \end{center}
% \vskip -0.2in
% \end{figure*}

To this end, we use recent results on private kernel density estimation. 
KDE is a common way to model a finite dataset in $\R^d$ as a probability distribution, which can then be sampled from. 
Recently, \cite{wagner2023fast} gave a DP KDE method for high-dimensional data. 
To us this suggests the following plan: 
\begin{enumerate}
    \item Take a public vocabulary (say, an English dictionary);
    \item Embed it in $\R^d$ using pre-trained term embeddings (following \cite{mikolov2013efficient});
    \item Construct a DP KDE distribution over vocabulary terms from the private dataset $\mathcal D$, using \cite{wagner2023fast};
    \item Draw samples of terms from the DP KDE distribution with which to seed the LLM prompts. 
\end{enumerate}   
The sampled terms would capture $\mathcal D$, since they are drawn from a distribution that approximates its KDE, while also satisfying differentially privacy.

This plan entails several challenges. 
For one, while \cite{wagner2023fast} gave an efficient way to privately estimate the KDE at any given point in $\R^d$, they did not provide a way to efficiently draw samples from the DP KDE distribution, and the problem appears infeasible in high dimensions. 
Our observation here is that the vocabulary, while large, is still feasible to enumerate over linearly. Thus, we can draw a term from the DP KDE distribution by querying the KDE of each term in the vocabulary, and then sampling from the resulting multinomial distribution.

Beyond sampling a single term, we wish to sample a sequence of terms. Much of the signal we want to capture in $\mathcal D$ may lie not just in individual terms, but in the joint distribution of co-occurring terms, e.g., ‘lung’ and ‘pulmonary embolism’ and ‘shortness of breath’. While enumerating over a large vocabulary to sample one keyphrase is feasible, enumerating over $k$-tuples to sample a sequence becomes infeasible, even for small values of $k$. 
To address this, we develop ways to sequentially sample dependent keyphrases, using \emph{ensembles} of DP KDEs.

By prompting the LLM with a DP sequence of keyphrases, a text record is privately generated.  A final consideration is that the text may not be written in the style of the data recipient (say, a different hospital).  
To overcome this, the final step in our method employs off-the-shelf domain adaption.
Ultimately, the quality of a set of synthetic texts is quantified by how effectively an external party can train an ML model.  Our experiments demonstrate how an ML model trained on a private synthetic text corpus generated by \alg\ can achieve high accuracy performance.

%\begin{figure*}[ht]
%\vskip 0.2in
%\begin{center}
%\centerline{\includegraphics[width=\textwidth]%{figures/fig_idea2.pdf}}
%\caption{ General idea behind the approach. We start with a collection of private documents.  Next, a private sequence of keyphrases is generated (the key technical contribution). The keyphrases are then used to prompt an LLM for a document containing the keyphrases.  This process is repeated multiple times to obtain a collection of private synthetic documents.}
%\label{fig:general_idea}
%\end{center}
%\vskip -0.2in
%\end{figure*}

%\begin{figure*}[t]
%\vskip 0.2in
%\begin{center}
%\centerline{\includegraphics[width=\textwidth]%{figures/fig_idea2.pdf}}
%\caption{\textbf{\alg}~general approach. A privately generated sequence of keyphrase is used to seed an LLM prompt for generating each synthetic text document (e.g., a medical record).\label{fig:method2}\\
%\\}
%\centerline{\includegraphics[width=\textwidth]{figures/fig_method}}
%\caption{\textbf{\alg}~detailed method overview. Color coding: \textcolor{methodred}{red -- private data}, \textcolor{methodpurple}{purple -- differentially privatized data (safe to release)}, \textcolor{methodgreen}{green -- public data}, \textcolor{methodblue}{blue -- public pre-trained model}. The pre-trained models are only used for inference, and on already privatized data.}\label{fig:flow}
%\end{center}
%\vskip -0.2in
%\end{figure*}

\begin{figure*}[t]
\vskip 0.2in
\begin{center}
\centerline{\includegraphics[width=\textwidth]{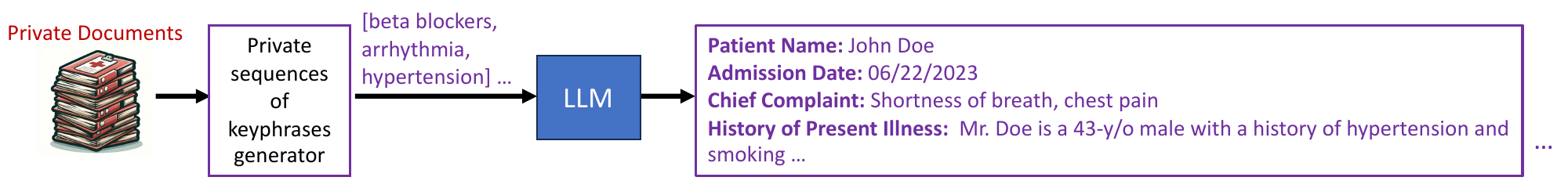}}
\caption{\textbf{\alg}~general approach. A privately generated sequence of keyphrase is used to seed an LLM prompt for generating each synthetic text document (e.g., a medical record).\label{fig:method2}}
\end{center}
\vskip -0.2in
\end{figure*}

\subsection{Related Work}\label{sec:related}

%\noindent\textbf{Differentially private training and fine-tuning.}

Recent years have seen ample and successful work on DP training of large deep generative models, including \cite{xie2018differentially,torkzadehmahani2019dp,chen2020gs,ramaswamy2020training,cao2021don,yu2021differentially,li2021large,zhang2023dpzero,tang2024private}, with some work specifically on private text generation (e.g., \cite{torfi2022differentially,yue2022synthetic,tang2023privacy}).
While effective, these methods generally require the ability to train (or fine-tune) the generative model, and allow it to directly access the private data. In our setting, where the data owner has only inference access to the LLM, and is not allowed to feed it the private data in the clear, these methods are not applicable. The line of work of \cite{harder2021dp,vinaroz2022hermite} is closer to ours, in that the generative model is trained only on data which is already privatized. However, the method still requires training the model, and with a specially designed objective function, which is inapplicable to LLMs. 

Private Aggregation of Teacher Ensembles (PATE) \cite{papernot2016semi,papernot2018scalable} is a prominent paradigm for DP machine learning, with recent extensions specialized to generative models \cite{jordon2018pate} and LLMs 
\cite{duan2023flocks}. It is based on private knowledge transfer from the private dataset to public data, which can then be used to train models without further privacy loss. The recent PromptPATE method \cite{duan2023flocks} realizes this paradigm in a way that only accesses an LLM with already privatized prompts, rendering it similar in spirit to ours. However, like prior PATE methods, it requires access to public data of a similar domain (though possibly differently distributed), to which private knowledge can be transferred.
In certain application contexts such data may be unavailable, for example in medical applications, where even public datasets released for research purposes are scarce and highly regulated. 

%a key limitation of this method (as well as prior PATE methods) is that it requires an available public dataset sufficiently similar to the private one, to which private knowledge can be transferred. In many scenarios, such public data is not necessarily available.

Most directly related to us is the recent concurrent work on Augmented Private Evolution (AugPE) for DP text generation~\cite{api2}, adapting a precursor work for private images~\cite{lin2023differentially}. It addresses the same problem setting as ours. Our experimental section thoroughly compares the two methods. 

%Two very recent/concurrent works seem closest to ours. One is Private Evolution (PE) \cite{lin2023differentially}, which also studied generating synthetic DP data with interface-only access to a pre-trained generative model, and suggested a method that ``evolves'' the pre-trained generative distribution toward the private data by iterative prompts. However, they instantiate and test their method only for image generation via pre-trained diffusion models. The other relevant work is HotPATE \cite{cohen2023hot}, which extends PATE to privately learning a distribution over responses, by iteratively prompting a pre-trained LLM with prompts constructed from the private data. However, their prompts are not privatized (rather, their focus is on aggregating an ensemble of private distributions while preserving both privacy and diversity), and thus their method allows the LLM direct access to the private data, which we consider prohibited. 

\begin{figure*}[t]
\vskip 0.2in
\begin{center}
\centerline{\includegraphics[width=\textwidth]{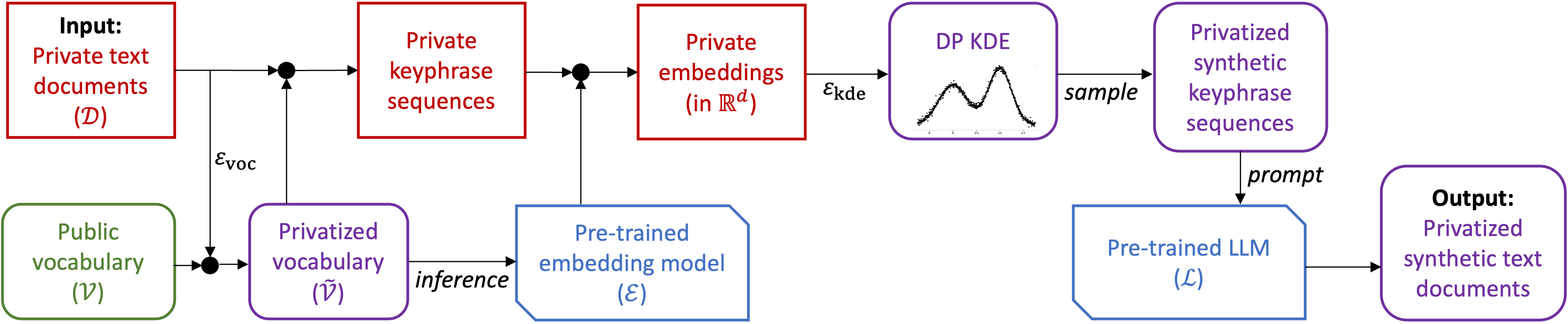}}
\caption{\textbf{\alg}~detailed method overview. Color coding: \textcolor{methodred}{red -- private data}, \textcolor{methodpurple}{purple -- differentially privatized data (safe to release)}, \textcolor{methodgreen}{green -- public data}, \textcolor{methodblue}{blue -- public pre-trained model}. The pre-trained models are only used for inference, and on already privatized data.}\label{fig:flow}
\end{center}
\vskip -0.2in
\end{figure*}
\subsection{Preliminaries}\label{sec:prelim}

\noindent\textbf{Differential Privacy.}
%We review the definition of DP \cite{dwork2006calibrating}. 
Let $\mathbb D$ be a ``universe'' of data records. Two datasets $\mathcal{D},\mathcal{D'}\subset\mathbb D$ are called \emph{adjacent} if one is obtained from the other by dropping a single record. 
Let $M$ be a randomized algorithm that takes a dataset $\mathcal{D}\subset\mathbb D$ as input. $M$ is said to be \emph{$(\varepsilon,\delta)$-DP} if for every adjacent $\mathcal{D},\mathcal{D'}\subset\mathbb D$, and every set $\mathcal S$ of possible inputs of $M$, it holds that
\[ \Pr[M(\mathcal{D})\in \mathcal S] \leq e^\varepsilon \Pr[M(\mathcal{D'})\in \mathcal S] + \delta . \]
Intuitively, the output of $M$ is insensitive to the inclusion of any single data record in the input dataset, and thus information about any single record should not be possible to glean from its output. 
In our case, $\mathbb D$ is the universe of all possible text documents, and $\mathcal{D},\mathcal{D'}$ are datasets of text documents. Thus, DP means that the output should be insensitive to the presence of any single text document.

The case $\delta=0$ is called ``pure'' DP, and in that case $M$ is said to be \emph{$\varepsilon$-DP}. Our method is pure DP. 
%Straightforward extensions to non-pure DP are possible, but a drop-in replacement of the $\varepsilon$-DP building blocks we use with their $(\varepsilon,\delta)$-DP analogs.

\noindent\textbf{Kernel density estimation.}
Given a set $X\subset\R^d$, the Gaussian KDE function $KDE_X:\R^d\rightarrow[0,1]$ is defined as $KDE_X(y)=\tfrac{1}{|X|}\sum_{x\in X}e^{-\norm{y-x}_2^2}$. Up to normalization, $KDE_X(y)$ is the density at $y$ of a mixture of Gaussians centered at each $x\in X$. This is a common way to model a finite dataset $X$ as a distribution over all $\R^d$. 

A DP mechanism for estimating $KDE_X(y)$ at any given $y\in\R^d$, up to a bounded error, was given in \cite{wagner2023fast}. It runs in time linear in the dataset size $|X|$ and in the dimension $d$. It can thus be used for high-dimensional embeddings obtained from pre-trained deep learning models.

\section{\alg: Method Description}\label{sec:method_main}

\begin{figure}
    \begin{center}
        \includegraphics[width=\textwidth]{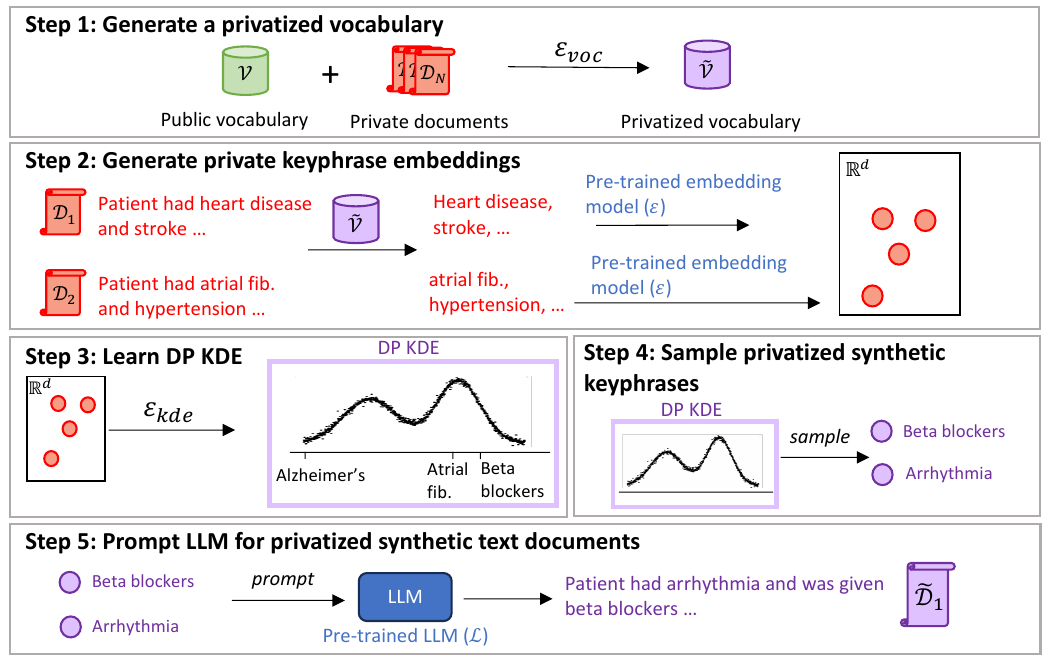}
        \captionsetup{font=footnotesize}
        \caption{\textbf{\alg} method overview. We illustrate the process of generating privatized synthetic medical records. For simplicity, each medical record is represented by a single sentence. The two example private documents shown here contain terms related to heart conditions. This results in the DP KDE having a higher concentration near the words related to the heart, and  hence sampling a term such as `beta blocker' for the synthetic key-phrase.} 
        %Color coding: \textcolor{methodred}{red -- private data}, \textcolor{methodpurple}{purple -- differentially privatized data (safe to release)}, \textcolor{methodgreen}{green -- public data}, \textcolor{methodblue}{blue -- public pre-trained model}. }
        \label{fig:detailedmethod}
    \end{center}
\vspace{-0.15 in}
\end{figure}

The core of \alg\ is prompting the LLM to generate synthetic text, by seeding the prompt with a sequence of keyphrases that are already privatized. 
We now describe how we instantiate this idea in detail.
The input to \alg\ consists of the following:
\begin{itemize}
 \item A private dataset $\mathcal{D}$ of text documents.
 \item A large public vocabulary $\mathcal{V}$. It need not be specialized to $\mathcal{D}$, but should be suitable for the data domain (for example, a glossary of medical terms for medical records, or a full English dictionary for Wikipedia articles).
 \item A public pre-trained embedding model $\mathcal{E}$, that can map terms from $\mathcal{V}$ into embeddings in $\R^d$. 
 \item Prompting access to an LLM $\mathcal L$. It can be prompted, but it is not allowed to see the private data, and therefore the prompts it receives must be already privatized.
\end{itemize}
The steps of \alg~are described next and illustrated in \Cref{fig:detailedmethod}.

%\subsection{Pre-processing: Privatized Vocabulary}\label{sec:method_prep}
\noindent\textbf{Pre-processing: Privatized vocabulary.}
We first expend some of the privacy budget on producing a limited privatized vocabulary $\widetilde{\mathcal{V}}$ from $\mathcal{V}$ (Step 1 in \Cref{fig:detailedmethod}). This serves two purposes: limiting the vocabulary size for better computational efficiency, and rendering it more relevant to the private dataset $\mathcal{D}$. We denote the privacy budget we expend here by $\varepsilon_{\mathrm{voc}}$.
We produce $\widetilde{\mathcal{V}}$ using vanilla DP histograms \cite{dwork2014algorithmic}; see appendix for details.
%\begin{CompactEnumerate}
% \item From each text document in $\mathcal{D}$, extract that first $S$ terms that appear in $\mathcal{V}$. %(Our implementation uses $S=10$.)
% \item Build a privatized histogram $\widetilde{\mathcal{H}}$ over $\mathcal{V}$ from the $S\cdot|\mathcal{D}|$ extracted terms, by adding an i.i.d.~sample from $\mathrm{Laplace}(S/\varepsilon_{\mathrm{voc}})$ to each count.
% \item The private vocabulary $\widetilde{\mathcal{V}}$ consists of the $N$ terms from $\mathcal{V}$ with the highest counts in $\widetilde{\mathcal{H}}$.
%\end{CompactEnumerate}
%It is easily seen that the $\ell_1$-sensitivity of the histogram is $S$, and therefore, the DP Laplace mechanism \cite{dwork2014algorithmic} ensures that $\widetilde{\mathcal{V}}$ is $\varepsilon_{\mathrm{voc}}$-differentially private.

%\subsection{Generating Private Keyphrase Sequences}\label{sec:method_main}
\noindent\textbf{Generating private keyphrase sequences.}
The main step of \alg\ is generating private sequences of keyphrases from $\widetilde{\mathcal{V}}$, that capture the private dataset $\mathcal{D}$ while maintaining differential privacy. 
We denote the privacy expended on this step by $\varepsilon_{\mathrm{kde}}$. Together with the preprocessing step, by the differential privacy composition theorem \cite{dwork2014algorithmic}, the total privacy budget of \alg\ is $\varepsilon_{\mathrm{voc}}+\varepsilon_{\mathrm{kde}}$.

We start by embedding all individual keyphrases $v\in\widetilde{\mathcal{V}}$ into vectors in $\R^d$, using the pre-trained embedding model $\mathcal{E}$ (Step 2 in \Cref{fig:detailedmethod}). The embedding dimension $d$ is typically high (e.g., $d=768$ is a common choice), and therefore subsequent operations must run in time at most linear in $d$.

Next, let $L$ be the desired number of keyphrases in an output sequence. Our current goal is to generate a collection of sequences from $\widetilde{\mathcal{V}}^L$ that would successfully represent $\mathcal{D}$ in downstream tasks. 
%We describe three approaches to generating sequences privately, and later discuss their different merits and test all of them in our experiments. 
%
%All three methods are based on constructing 
To this end, we will build a collection of suitable DP-KDE distributions over the private data (Step 3 in \Cref{fig:detailedmethod}), and generate private sequences from the associated DP-KDE scores (Step 4 in \Cref{fig:detailedmethod}). 
This is the most technically involved step in our method, and we explain it in stages: first, how to sample one keyphrase, and then, how to generate a sequence of keyphrases. 

\noindent{\emph{(i) Sampling one keyphrase.}}
To sample a single keyphrase, we use the high-dimensional DP-KDE mechanism from \cite{wagner2023fast}. It builds the KDE data structure on a given set of vectors $X\subset\R^d$ in time $O(d|X|)$, and then allows querying the KDE score of any point in $\R^d$ in time $O(d)$. 
%For a private dataset of vectors $V\subset\R^d$ and desired accuracy $\alpha$, it builds the DP-KDE data structure in time $O(d|V|/\alpha^2)$, and then allows querying the KDE score of any point in $\R^d$ in time $O(d/\alpha^2)$ up to additive error $\alpha$.
%
However, in order to generate a sequence, we need to draw sample of new points from the DP-KDE distribution, rather than query the scores at given points. 
Unfortunately, \cite{wagner2023fast} did not present an efficient sampling algorithm for their mechanism. All known sampling methods take time exponential in $d$ --- e.g., by enumerating over a $d$-dimensional grid --- which is infeasible for high-dimensional embeddings. 

To resolve this issue, 
we exploit the fact that we are only interested in samples from our privatized vocabulary $\widetilde{\mathcal{V}}$, rather than from all of $\R^d$. 
Thus, $\widetilde{\mathcal{V}}$ can serve as a feasible replacement for grid enumeration. 
We query the DP-KDE for the density $\tilde p_v$ of every keyphrase $v\in\widetilde{\mathcal{V}}$, which entails a feasible running time of $O(d|\widetilde{\mathcal{V}}|)$. Then, we may draw samples from the resulting multinomial distribution over $\widetilde{\mathcal{V}}$, wherein each $v\in \widetilde{\mathcal{V}}$ is sampled with probability $\tilde p_v/\sum_{v'\in\widetilde{\mathcal{V}}}\tilde p_{v'}$.

%For a private dataset of vectors $V\subset\R^d$ and desired accuracy $\alpha$, it builds the DP-KDE data structure in time $O(d|V|/\alpha^2)$, and then allows querying the KDE score of any point in $\R^d$ in time $O(d/\alpha^2)$ up to additive error $\alpha$. 
%One difficulty is that in order to generate a sequence, we need to draw a sample from the DP-KDE distribution, whereas unfortunately, \cite{wagner2023fast} did not present an efficient sampling algorithm for their high-dimensional mechanism. All known sampling methods take time exponential in $d$, which is infeasible for high-dimensional embeddings. To circumvent this issue, we will query the DP-KDE data structure for every embedding of a keyphrase in $\widetilde{\mathcal{V}}$, which entails a feasible running time of $O(d|\widetilde{\mathcal{V}}|/\alpha^2)$, and then draw samples from the resulting multinomial distribution over $\widetilde{\mathcal{V}}$, wherein each keyphrase is sampled with probability proportional to its DP-KDE score.

\noindent{\emph{(ii) Generating a sequence of keyphrases.}}
There are two natural ways to generate a sequence of $L$ keyphrases with the above procedure for sampling one keyphrase. We explore both of them.

The first method is to concatenate $L$ single keyphrase samples drawn independently. 
One advantage of this method is its fast running time. One possible downside is that it may fail to capture correlations and dependencies between phrases that tend to occur together in texts, which may be important in downstream learning. Yet, as our experiments will show, this method can be highly effective.

The second method is to generate the sequence iteratively, where each new keyphrase is added to the sequence while taking into account the prefix of keyphrases that were already generated. In more detail, iterative sequence generating proceeds as follows:

\begin{enumerate}
 \item[1.] Initialize an empty sequence of keyphrases $P$.
 \item[2.] For $i=1,\ldots,L:$
 \begin{enumerate}
 \item[2.1] For every single keyphrase in $w\in\widetilde{\mathcal{V}}$, let $Pw$ be the concatenation of $P$ and $w$. Thus $Pw$ is a keyphrase sequence of length $i$. Enumerate over all possible concatenations $\{Pw:w\in\R^d\}$, and compute the DP-KDE score of each.
 \item[2.2] Choose among them a high scoring sequence $Pw^*_i$. % (either by sampling proportionally to the scores, or by taking the highest scoring sequence).
 \item[2.3] $P\leftarrow Pw^*_i$.
 \end{enumerate}
 \item[3.] Return the output sequence $P$.
\end{enumerate}
In each iteration, this scheme makes only $|\widetilde{\mathcal{V}}|$ queries to the KDE data structure, for a total of $L|\widetilde{\mathcal{V}}|$, rendering its running time feasible. At the same time, jointly scoring keyphrase prefixes rather than just individual keyphrases promotes preserving dependencies between co-occurring keyphrases. 

One downside of this method is that it queries the KDE scores for keyphrase prefixes of multiple lengths --- that is, of vectors with varying dimensions --- and thus requires multiple DP KDE data structures. The privacy budget needs to be partitioned between these data structure, leading to decreased accuracy in each. Na\"ively, the method requires $L$ DP KDE data structures, one per each prefix length. In \Cref{sec:seqgen_appendix} we show how this can be improved to $O(\log L)$.

\noindent\textbf{Prompting the LLM and post-processing.}
Once the privatized keyphrase sequences have been generated, we query the LLM once per sequence to generate a corresponding synthethic text document (Step 5 in \Cref{fig:detailedmethod}). The prompt can be tweaked in various ways to improve the output quality for specific applications, but since our focus here is on developing a general method, we restrict our scope to generic prompts that only mention the desired document type (e.g., ``write a medical record that contains the following terms: $\langle$\textit{sequence of keyphrases}$\rangle$''). Since the keyphrase sequences are $(\varepsilon_{\mathrm{voc}}+\varepsilon_{\mathrm{kde}})$-DP, then by DP post-processing \cite{dwork2014algorithmic}, the LLM output is $(\varepsilon_{\mathrm{voc}}+\varepsilon_{\mathrm{kde}})$-DP as well.

The privatized documents generated by the LLM are intended for use in downstream ML by a client with its own test set. 
One issue is that the LLM output might look very different than the test set. For example, if we prompt the LLM to generate a ``medical record'', it might generate one in a completely different format than those in a client hospital's test set.
To resolve this, the final step in \alg\ is off-the-shelf domain adaptation on the client side, between the synthetic data and the test data. 

Another possibility we consider is \emph{few-shot prompting}, where the client provides a small number of examples from their test set (say, a few example medical records with the desired format), which can augment the prompt, so that the LLM is prompted to produce outputs adhering to this format. This fits application scenarios where the client can communicate with the curator during the generation of the private data. We examine this in \Cref{app:fewshot}.

\section{Experiments}\label{sec:experiments}

\begin{table*}[t]
\centering
{\renewcommand{\arraystretch}{2}
\caption{Datasets, public vocabularies and pre-trained embedding models used in our experiments.} \label{tbl:datasets}
\begin{centering}
\small
\begin{tabular}{llcll}
\toprule
 \textbf{Dataset} & \textbf{Data type} & \textbf{\# Classes} & \textbf{Public vocabulary} ($\mathcal V$) & \textbf{Pre-trained emb.~model} ($\mathcal E$) \\
\midrule
MIMIC & medical records & 2 & UMLS medical glossary & BioBERT ($d=768$) \\
DBPedia-14 & Wikipedia summaries & 14 & GloVe English dictionary & SentenceBERT ($d=768$) \\
\bottomrule
\end{tabular}
\end{centering}}
\end{table*}

%\subsection{Datasets}
%\noindent\textbf{Datasets.}
We test \alg\ on two text classification tasks on public datasets from different domains (cf.~\Cref{tbl:datasets}):

\noindent\textbf{(i) MIMIC} \cite{goldberger2000physiobank}: MIMIC is a dataset of medical records. We consider the binary classification task of patients diagnosed with cardiac conditions versus those who weren't, from the description of the hospital course provided in their discharge summary. The groundtruth labels are determined by the ICD diagnostic codes\footnote{\url{https://www.who.int/standards/classifications/classification-of-diseases}} 
%\cite{ICDLink}
for which the patients were billed.
We chose 10K medical records of each class at random, and use 8K of each class as the private dataset $\mathcal D$, and the remaining records as the test set. 
As the public vocabulary $\mathcal V$ for this task, we use 400K medical concepts from the Unified Medical Language System (UMLS) Glossary.\footnote{\url{https://www.nlm.nih.gov/research/umls/index.html}}
%\cite{UMLSLink}
For the pre-trained embedding model $\mathcal E$ we use BioBERT \cite{lee2020biobert}. See \Cref{app:mimic} for additional details.

\noindent\textbf{(ii) DBPedia-14} \cite{NIPS2015_250cf8b5}: The dataset consists of short summaries of Wikipedia articles from 14 topic categories, and the task is topic classification. We use its given split to training and test sets, treating the training set as the private dataset $\mathcal D$. 
For the public vocabulary $\mathcal V$, we use GloVe 6B dictionary \cite{pennington2014glove},\footnote{\url{https://nlp.stanford.edu/projects/glove/}} which consists of 400K generic English language words. For the pre-trained embedding model $\mathcal E$ we use ``all-mpnet-base-v2'' from SentenceBERT \cite{reimers2019sentence}.

\noindent\textit{Inclusion in LLM pre-training data.} 
LLMs typically do not disclose the data they have been pre-trained on, which creates methodological difficulties in using them for research, particularly in the context of privacy, so we comment on this matter directly. In our case, the MIMIC legally binding dataset usage terms prohibit using the data in training LLMs,\footnote{\url{https://physionet.org/news/post/gpt-responsible-use}} 
and therefore we may assume with high certainty that the main LLM we use in experiments (Claude v2) has not seen the private data.\footnote{We have also verified this with Anthropic AI by contacting their Support Team.} 
On the other hand, it is rather likely to have seen DBPedia-14, and nearly certainly has seen the original Wikipedia articles on which the dataset is based. 
As the case would be similar with any standard benchmark dataset, we opt to include one for completeness while acknowledging this limitation. 

\noindent\textbf{Experimental setup.}
For both datasets, our downstream task is text classification. To apply \alg\ to this task, in the preprocessing step we create a mutual privatized dictionary $\widetilde{\mathcal{V}}$ for all classes, and then perform the sequence generation step for each class separately with its own DP KDE. Using each class DP KDE, we generate keyphrase sequences for seeding prompts, and the texts the LLM generates for them are used as synthetic training records labeled with the same class. 
The LLM prompts in the final step contain only the overall document type (``medical record'' for MIMIC and ``summary of a Wikipedia-style article'' for DBPedia-14) and the privatized sequence of phrases, without explicit class information (e.g., we do not prompt the LLM to generate medical records for patients with cardiac conditions, or Wikipedia articles about artists). 

\noindent\textbf{Method components.}
In the preprocessing step we use $S=10$ phrases of each dataset record to produce a private vocabulary $\widetilde{\mathcal{V}}$ of size $1000$ from the public vocabulary $\mathcal V$. In the sequence generation step, we generate $1500$ sequences of length $L=10$ for each MIMIC class, and $1000$ sequences of length $L=10$ for each DBPedia-14 class. The main LLM we use is Anthropic AI's Claude v2 \cite{ClaudeLink}.
%through its publicly available prompting API. 
%As a point of reference, and for compatibility with prior work, we also include experiments with OpenAI's GPT-2 \cite{radford2019language}, a smaller LLM whose parameters are publicly available. 
For the domain adaptation step, we use Deep CORAL \cite{sun2016deep}, see \Cref{sec:deep_coral_details} for details.

\noindent\textbf{Privacy parameters.}
To choose $\epsilon$ we followed the guidelines offered by \cite[Section 5]{ponomareva2023dp} for private ML. They review current DP-trained ML models and LLMs, as well as real-world deployments of DP, and report them using $\varepsilon$ between 5 to 15. They advocate for $\epsilon\leq10$ for real-world deployments.

We chose two values, one representing tighter privacy and one moderate privacy, for each of the two DP budget components in \alg: $\varepsilon_{\mathrm{voc}}\in\{1,5\}$ and $\varepsilon_{\mathrm{kde}}\in\{5,10\}$. 
We evaluate \alg\ with total privacy budgets equal to the resulting four combinations, $\varepsilon=\varepsilon_{\mathrm{voc}}+\varepsilon_{\mathrm{kde}}\in\{6,10,11,15\}$. 
They represent tighter ($6$), moderate ($10,11$) and looser ($15$) privacy settings. 
We limit our experiments to four $\epsilon$ values due to the high impact associated with large-scale experiments with a high-end LLM. 

%concrete examples, they mention that the US Census (likely the most famous and celebrated real-world deployment of DP) used $\varepsilon=12.2$, and that the Gboard Spanish LLM (the only publicly deployed example of a DP LLM they could find) used $\varepsilon=8.9$. Their review of the academic literature yields that state-of-the-art DP-trained LLMs (as of 2023) have used $\varepsilon$ between $5.36$ to $6.7$, which resulted in a performance drop of between 3% to 34% depending on the model and on work cited. The choices of $\varepsilon_{total}$ in our experiments are in line with these ranges.

\noindent\textbf{Comparison to AugPE.} 
AugPE \cite{api2} is another recent method for generating private synthetic texts by LLM prompts (see \Cref{sec:related}). It start by prompting the LLM for random texts, and then sequentially prompts it to ``evolve'' them toward the private texts while maintaining DP.  
As implemented and reported in \cite{api2}, the initial prompt includes the target class topic in the clear (e.g., ``write an article about an athlete''), and also includes random words from a list that the LLM had been prompted to generate for that class topic (e.g., generic words related to athletes). 
This already ``gives away'' much of the intended content to the LLM in the clear.
However, in our setting, the class topics are considered private and are unknown to the algorithm, and therefore cannot be included in prompts. Furthermore, including the class topic in the prompt may lead to methodological artifacts, as it is hard to discern how much of the downstream accuracy owes to the class topic, which was given to the algorithm free of privacy, versus meaningfully making use of the private dataset. 

Therefore, to directly compare AugPE to \alg\ in our setting, we do not include the class topic in prompts for neither method. Since also removing the random words from AugPE's prompts leads to substantially degraded performance (see results below), we replace them  with random words from the privatized dictionary $\widetilde{\mathcal{V}}$ that we create in our pre-processing step (with $\varepsilon_{\mathrm{voc}}=1$). 
This is a variant of AugPE compatible with the privacy constraints of our setting, and we denote it by AugPE+$\widetilde{\mathcal{V}}$.

\subsection{Results}
\begin{table*}[t]
{

\small
\begin{center}
% TAL: Please don't move table captions, the NeurIPS style guide says they need to be ABOVE the table
\caption{Classification accuracy of \alg\ (ours), AugPE+$\widetilde{\mathcal{V}}$, and the original texts.}\label{tbl:mainres}
%\caption{\textcolor{red}{Classification accuracy of proposed \alg\ method vs. AugPE. 
%In this table, we fix the size of the generated content to 1K examples for all methods.  (In Figure~\ref{fig:result_augpe_epsall}, larger sample sizes for AugPE are considered.) The rows depict increasing privacy budgets. The columns compare the two methods on the MIMIC and DBPedia datasets.  The last row shows the accuracy of the model on the original data, and hence is as an accuracy upper bound. }}
%\begin{adjustwidth}{-0.15in}{-0.1in}
\begin{tabular}{ c | c c c | c c c }
\toprule
& \multicolumn{3}{c|}{MIMIC} & \multicolumn{3}{c}{DBPedia-14} \\
\cmidrule{2-7}
%\midrule
  \hfill Method: & \alg\ & AugPE+$\widetilde{\mathcal{V}}$  & AugPE+$\widetilde{\mathcal{V}}$ & \alg\ & AugPE+$\widetilde{\mathcal{V}}$  & AugPE+$\widetilde{\mathcal{V}}$ \\
 \hfill\textit{\footnotesize{(\#texts, \#prompts) / class:}} & \textit{\footnotesize{(1K, 1K)}} & \textit{\footnotesize{(100, 1K)}} &\textit{\footnotesize{(1K, 10K)}} &\textit{\footnotesize{(1K, 1K)}} &\textit{\footnotesize{(100, 1K)}} &\textit{\footnotesize{(1K, 10K)}} \\
\midrule
%\rule{0pt}{3ex}
\hspace{-5pt}$\varepsilon=6$\scriptsize{$\;\;(=1+5)\hfill$} & 71.6\% &50.0\% &63.8\% & 77.9\% &45.7\% &79.3\%  \\
\hspace{-5pt}$\varepsilon=10$\scriptsize{$\;\;(=5+5)\hfill$} & 71.7\% &50.0\% &60.8\% & 80.4\% &46.7\% &76.3\% \\
\hspace{-5pt}$\varepsilon=11$\scriptsize{$\;\;(=1+10)\hfill$} & 72.2\% &50.3\% &66.3\% & 82.4\% &38.7\% &79.4\% \\
\hspace{-5pt}$\varepsilon=15$\scriptsize{$\;\;(=5+10)\hfill$} & 72.2\% &50.0\% &72.3\% & 85.1\% &34.8\% &80.5\% \\
\midrule
Original training set 
& \multicolumn{3}{c|}{76.0\%} &  \multicolumn{3}{c}{97.0\%} \\
\bottomrule
\end{tabular}
%\end{adjustwidth}
\end{center}}
\end{table*}

\noindent\textbf{Downstream classification accuracy.}
\Cref{tbl:mainres} lists the accuracy obtained  by training a classifier over the synthetic corpora produced by \alg,
compared to AugPE+$\widetilde{\mathcal{V}}$, and to training on a similar number of (randomly sampled) texts from the original training set. 
For each corpus we train the classifier by finetuning the pre-trained embedding model (BioBERT for MIMIC and SentenceBERT for DBPedia-14), namely by training a 3-layer fully connected neural network composed over it. 

Note that while \alg\ prompts the LLM once per synthetic text record, AugPE makes 10 iterative prompts per text, due to its ``evolution'' procedure (see also \Cref{sec:augpe_details}). To compare the methods, for \alg\ we use 1K docs and 1K prompts per class, and for AugPE+$\widetilde{\mathcal{V}}$, we report in \Cref{tbl:mainres} results once with the same number of prompts (and less texts), and once with the same number of texts (and more prompts). In \Cref{fig:result_augpe_epsall} we further compare the accuracy of both methods under varying ``prompts budgets'' for AugPE+$\widetilde{\mathcal{V}}$, and also report results for AugPE (without $\widetilde{\mathcal{V}}$).

The results show the corpora generated by \alg\ preserve much of the predictive accuracy of the original texts, while satisfying DP. Compared to either variant of AugPE, \alg\ achieves better accuracy with fewer prompts to the LLM, about $\times10$ less prompts for comparable accuracy. 
Furthermore, \alg\ is more consistent in improving as the privacy budget $\varepsilon$ is increased.

\noindent\textbf{Additional experiments.} More experimental results and ablations are reported in \Cref{sec:experiments_appendix}.

\noindent\textbf{\alg\ Examples.}
\Cref{tbl:example1} contains examples of \alg\ for three classes from DBPedia-14 (examples for all classes are given in  \Cref{tbl:example1full,tbl:example2} in the appendix). For each class we show two seed phrase sequences generated by \alg, and the LLM output when prompted to ``generate a summary of a fictitious wikipedia-style article'' that contains them. The intended class is not specified in the prompt, and for each class we show one ``good'' example, where we (subjectively) judge the topic of the output to be of the intended class, and one ``bad'' example where it is not. %Most outputs are ``good''. %, leading to the downstream accuracy reported above. 

Examples for MIMIC synthetic records are not included here due to an abundance of caution in complying with the dataset's legal terms of use, which require a certification for viewing the data.

\begin{figure}[t]
    \begin{center}
        \includegraphics[width=\textwidth]{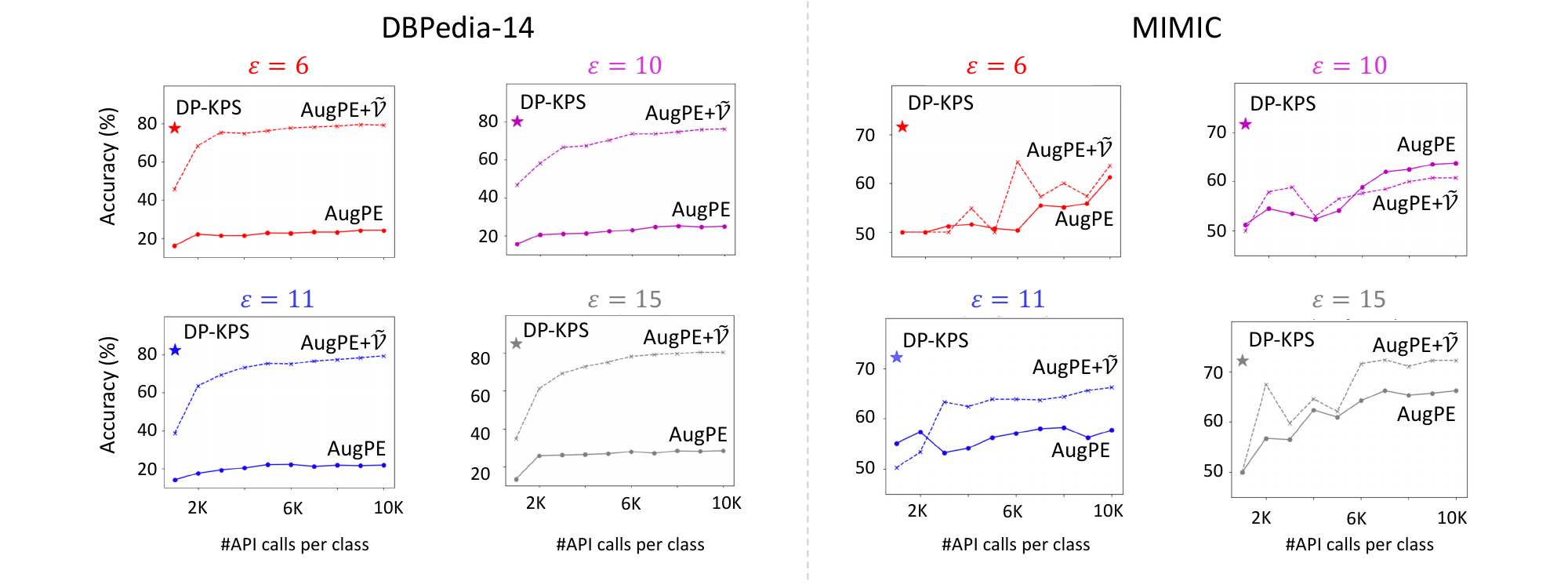}
        \caption{Classification performance of data generated by DP-KPS (ours - star marker), AugPE+$\widetilde{\mathcal{V}}$ (dotted line) and AugPE (solid line) with varying privacy and prompt budget.}% DP-KPS requires about one tenth the number of prompts required by AugPE+$\widetilde{\mathcal{V}}$ to reach a comparable performance. }
        \label{fig:result_augpe_epsall}
    \end{center}
\end{figure}

\begin{table*}
\caption{\textcolor{goodgreen}{Good} and \textcolor{badred}{bad} examples for classes 1--3 in DBPedia-14, with $\varepsilon_{\mathrm{voc}}=1$ and $\varepsilon_{\mathrm{kde}}=10$.} \label{tbl:example1}
{\renewcommand{\arraystretch}{2}
\begin{centering}
\setlength\extrarowheight{-3pt}
\small
%\begin{adjustwidth}{-0.5in}{-0.1in}
\begin{tabular}{p{0.08\linewidth}p{0.87\linewidth}}
\toprule
 Class & \alg\ results: Keyphrase sequence and corresponding LLM output \\
\midrule
Company &

['plants', 'company', 'distribution', 'operating', 'social', 'ships', 'andrew', 'based', 'michigan', 'york']: \newline 
\textcolor{goodgreen}{Plants Distribution Company is a Michigan-based company operating ships to distribute plants and other goods. Founded by Andrew York in New York, the small social enterprise distributes products across the region.} \\
&
['director', 'institution', 'royal', 'schools', 'municipality', 'club', 'journal', 'johnson', 'office', 'oil']: \newline 
\textcolor{badred}{John Johnson was appointed director of the Royal Schools institution in the municipality in 2020. He previously worked in the office of the Oil Club journal.} \\
\midrule

Educational institution
&
['1908', 'school', 'school', 'college', 'institute', 'school', 'germany', 'university', 'college', 'manufacturing']: \newline 
\textcolor{goodgreen}{The College Institute was founded in Germany in 1908 as a manufacturing school. Originally a vocational college, it later became a university specializing in engineering and technology. The school focuses on hands-on training in manufacturing techniques.} \\
&
['scott', 'school', 'brazilian', 'guard', 'sciences', '1996', 'high', 'lee', 'mystery', 'england']: \newline 
\textcolor{badred}{Scott Lee was born in England in 1996. He attended a Brazilian jiu-jitsu school in high school where he learned grappling and the martial art of Brazilian jiu-jitsu. The mystery sciences behind the techniques fascinated Scott during his time there.}\\
\midrule

Artist &

['species', 'february', '1944', 'european', '1929', 'music', 'musician', 'duo', 'frank', 'light']: \newline 
\textcolor{goodgreen}{The musical duo Frank \& Light was formed in February 1944 in Europe by Frank Smith and Light Williams. The duo, known for their unique blend of jazz and classical music, released their debut album "Species of Music" in 1929 to critical acclaim. Though popular in the late 1940s, they disbanded as a musical act in the early 1950s.} \\

& ['england', 'june', '1972', '2003', 'english', 'animated', 'nonfiction', 'publisher', 'utah', 'november']: \newline 
\textcolor{badred}{In November 2003, an English animated nonfiction publisher based in Utah released a book about events that took place in England in June 1972. The book was well-received upon its November 2003 publication. }\\

\bottomrule
\end{tabular}
%\end{adjustwidth}
\end{centering}}
\end{table*}

\section{Conclusion}
The question of how best to generate private synthetic data is acutely important -- particularly when hospitals want to build predictive models but cannot share medical records.  A method for generating private keyphrases is introduced, carefully combining high-dimensional embeddings with KDE and DP to generate sequences of phrases that capture the private text corpus. Experiments suggest that our presented method generates medical records that can be used to train an ML model with higher accuracy, stronger privacy and lower LLM budget than baselines. 

Exciting challenges lie ahead.  Medical records are richer than just text.  They contain numbers (from blood reports), time series (of heart rate, respiration rate, etc) and images (x-rays, MRI).  The question of how to privately and jointly generate these diverse components is an unsolved grand challenge.

\newpage

\bibliography{dp}
\bibliographystyle{amsalpha}

\newpage
\appendix
\section{Additional Method Details}

%\begin{figure*}[ht]
%\vskip 0.2in
%\begin{center}
%\centerline{\includegraphics[width=1.0\textwidth]{figures/fig_vocab}}
%\caption{Generating a Private Vocabulary. The party that holds the private collection of documents selects a public dictionary as the basis for the vocabulary shown in the x-axis of the middle panel. This party then creates the the actual histogram of counts from their private documents (middle panel). In order to make the histogram differentially private, the party adds noise to the counts (last panel). Finally, the most popular keyphrases are chosen to winnow down the vocabulary. In the figure, the two most popular keyphrases are selected.}
%\label{fig:vocab}
%\end{center}
%\vskip -0.2in
%\end{figure*}

\subsection{DP-KPS Vocabulary Privatization}
The privatized vocabulary $\widetilde{\mathcal V}$ is generated as follows. 
%See Figure~\ref{fig:vocab}
\begin{enumerate}
 \item From each text document in $\mathcal{D}$, extract that first $S$ terms that appear in $\mathcal{V}$. %(Our implementation uses $S=10$.)
 \item Build a privatized histogram $\widetilde{\mathcal{H}}$ over $\mathcal{V}$ from the $S\cdot|\mathcal{D}|$ extracted terms, by adding an i.i.d.~sample from $\mathrm{Laplace}(S/\varepsilon_{\mathrm{voc}})$ to each count.
 \item The private vocabulary $\widetilde{\mathcal{V}}$ consists of the $N$ terms from $\mathcal{V}$ with the highest counts in $\widetilde{\mathcal{H}}$.
\end{enumerate}
It is easily seen that the $\ell_1$-sensitivity of the histogram is $S$, and therefore, the standard DP Laplace mechanism (cf.~\cite{dwork2014algorithmic}) ensures that $\widetilde{\mathcal{V}}$ is $\varepsilon_{\mathrm{voc}}$-differentially private.

\subsection{Details on Sequence Generation Methods}\label{sec:seqgen_appendix}
As discussed in \Cref{sec:method_main}, the main step of DP-KPS generates private sequences of keyphrases from $\widetilde{\mathcal{V}}$. We now detail the methods to do this with their computational parameters. 

Let $L$ be the desired number of keyphrases.  Our sequence generation methods are based on constructing a collection of suitable DP-KDE distributions over the private data, and generating private sequences from the associated DP-KDE scores. We use the high-dimensional DP-KDE mechanism from \cite{wagner2023fast}. 
For a private dataset of vectors $V\subset\R^d$ and desired accuracy $\alpha$, it builds the DP-KDE data structure in time $O(d|V|/\alpha^2)$, and then allows querying the KDE score of any point in $\R^d$ in time $O(d/\alpha^2)$ up to additive error $\alpha$. 
To ease notation, in what follows we treat $\alpha$ as a small constant, and suppress it in asymptotic running time bounds.

\noindent\textbf{Independent keyphrase generation.}
A simple way to privately generate a length-$L$ sequence is to build a single DP-KDE distribution over $\widetilde{\mathcal{V}}$ using all single keyphrases from $\mathcal{D}$, and then draw $L$ i.i.d.~samples from this distribution to produce a length-$L$ sequence. 
The preprocsessing time of this method is $O(d|\mathcal{D}|M+d|\widetilde{\mathcal{V}}|)$, where $M$ is the maximum number of keyphrases in a document, to first construct the DP-KDE data structure, and then query the density of each embedding in $\widetilde{\mathcal{V}}$. Then, a length-$L$ sequence can be generated in time $O(L)$ by drawing $L$ i.i.d.~samples from the induced multinomial distribution over $\widetilde{\mathcal{V}}$.

While this method is fast, independent keyphrase sampling may fail to capture correlations and dependencies between often co-occurring keyphrases, which may be important in downstream learning tasks. Therefore, we next consider sequence generation methods that preserve such correlations.

\noindent\textbf{Iterative sequence generation.}
Ideally, we would have liked to draw a sample from a DP-KDE distribution over all of $\widetilde{\mathcal{V}}^L$. However, computing the induced multinomial distribution over $\widetilde{\mathcal{V}}^L$ would require $|\widetilde{\mathcal{V}}|^L$ queries to the DP-KDE data structure, which is prohibitive even for moderate values of $|\widetilde{\mathcal{V}}|$ and $L$. 
Instead, drawing on intuition from conditional sampling and auto-regressive models, we generate a sequence by the iterative process described in \Cref{sec:method_main}.

This scheme performs only $L|\widetilde{\mathcal{V}}|$ DP-KDE queries overall ($|\widetilde{\mathcal{V}}|$ queries in each of the $L$ iterations), rendering its running time feasible. 
Yet, a hurdle toward implementing it is that the DP-KDE mechanism operates on vectors of a fixed dimension, while here we need to query vectors of varying dimensions (the length-$i$ prefixes $Pw$ in iteration $i$ have dimension $di$). 
We describe how to handle this by using an ensemble of DP-KDE data structures. 

\noindent\emph{Linear size DP-KDE ensemble.}
One way is to simply create $L$ DP-KDE mechanism $K_1,\ldots,K_L$, where $K_i$ operates on dimension $di$. 
In the above sequence generation scheme, iteration $i$ would make $|\widetilde{\mathcal{V}}|$ DP-KDE queries to $K_i$. 

We calculate the computational parameters of this method. The privacy budget $\varepsilon_{\mathrm{kde}}$ needs to be allocated among the $K_i$'s, so each is constructed with privacy parameter $\varepsilon=\varepsilon_{\mathrm{kde}}/L$. Since the error of the DP-KDE data structure from \cite{wagner2023fast} degrades with $\varepsilon$ like $\sqrt{1/\varepsilon}$, the error of each $K_i$ degrades by $\sqrt{L}$ (compared to using a single DP-KDE as Method I does). 
The overall running time is $O(dL^2|\mathcal D|M)$ for building tge DP-KDEs (each $K_i$ takes time $O(di|\mathcal D|M)$ to build), and $O(dL^2|\widetilde{\mathcal{V}}|)$ to generate every length-$L$ sequence. This is considerably more expensive than independent sequence generation, but has the potential advantage of preserving keyphrase correlations.

\noindent\emph{Logarithmic size DP-KDE ensemble.}
To improve the error and computational cost of the linear DP-KDE ensemble from \Cref{sec:method_main}, we also propose a method more frugal in the number of DP-KDEs. 
Suppose we construct just a single DP-KDE data structure $K$, over the full length-$L$ sequence dimension, $dL$. In iteration $i<L$, to retrieve the DP-KDE score of a $di$-dimensional prefix $x\in\R^{di}$, we pad $x$ with zero blocks into a $dL$-dimensional vector $\bar x\in\R^{dL}$, and query $K$ on $\bar x$. If the vectors over which $K$ is constructed are normalized---which is the case for many widely used pre-trained embedding models---then we prove that zero padding returns an accurate DP-KDE estimate even for prefixes. This is formalized in \Cref{thm:ensemblekde} below.

Ostensibly, it now suffices to build just one DP-KDE. However, a small hurdle is that different prefix lengths also require different kernel bandwidth settings when computing their KDE. This is slightly technical, and we give the details below, but the upshot is that if too many blocks in $\bar x$ are the result of zero padding, then its KDE score would not be informative even if we were to compute it exactly.

To remedy this, we build logarithmically many DP-KDEs, $K_1,\ldots,K_\ell$, with $\ell=\lceil \log L \rceil$. 
Each $K_j$ is built for vectors of dimension $d\cdot2^j$. 
In the iterative sequence generation scheme, iteration $i$ makes its queries to $K_{\lceil \log i\rceil}$, by zero-padding them from dimension $di$ to dimension $d\cdot2^{\lceil \log i\rceil}$, which is between $di$ to $2di$. This ensures that at most half the query blocks are the result of zero padding, thus ensuring thar the DP-KDE score is informative.

We calculate the computational parameters of this method. 
The privacy budget $\varepsilon_{\mathrm{kde}}$ is allocated among the $K_j$'s, so each is constructed with privacy parameter $\varepsilon=\varepsilon_{\mathrm{kde}}/\ell$ (recall that $\ell=O(\log L)$), and thus the DP-KDE error of each $K_j$ in the logarithmic ensemble degrades by $\sqrt{\log L}$ (compared to $\sqrt L$ in \Cref{sec:method_main}). The time to build each $K_j$ is $O(d\cdot 2^j|\mathcal D|M)$, hence the total building time is $O(dL|\mathcal D|M)$ (compared to $O(dL^2|\mathcal D|M)$ for the linear size ensemble). The time to generate a sequence is $O(dL^2|\mathcal D|M)$, similarly to \Cref{sec:method_main}. 
Thus, the logarithmic DP-KDE ensemble has better accuracy and building time compared to the linear ensemble, with the same sequence generation time.

To state and prove our result for logarithmic ensemble formally, we introduce some notation. 
We consider a collection of embedding vectors in $\R^d$, and assume they all have the same squared norm, $u>0$. 
Furthermore we are interested in sequences of length $L$ of such embedding, which are vectors in $\R^{dL}$ of uniform length $uL$. For such a vector $x\in\R^{dL}$, which is say the concatenation of $L$ vectors $x_1,\ldots,x_L\in\R^d$, we will denote its length-$\ell$ prefix, for every $\ell=1,\ldots,L$, by $x^{[:\ell]}$, and its remaining suffix by $x^{[\ell:]}$. 
Thus, $x^{[:\ell]}$ is the vector in $\R^{d\ell}$ given by the concatenation of $x_1,\ldots,x_\ell\in\R^d$, and $x^{[\ell:]}$ is the vector in $\R^{d(L-\ell)}$ given by the concatenation of $x_
{\ell+1},\ldots,x_L$. Note that $\norm{x^{[:\ell]}}_2^2=u\cdot\ell$, and $\norm{x^{[\ell:]}}_2^2=u\cdot(L-\ell)$

To prove our result, we define an \emph{asymmetric} Gaussian kernel, $k^{[:\ell]}:\R^{dL}\times\R^{d\ell}\rightarrow[0,1]$, by
\[ k^{[:\ell]}(x,y) = \exp(- 
 \norm{x^{[:\ell]}-y}_2^2). 
\]
Note that this kernel measures similarity between two spaces of different dimensionality: the first input $x$ is from $\R^{dL}$, and the second input $y$ is from $\R^{d\ell}$. Of course, its value coincides with the usual (symmetric) Gaussian kernel over $\R^d$ by truncating $x$ to dimension $d\ell$, but this asymmetric notation will be useful for us below in handling multiple prefix lengths simultaneously.

Let $X\subset\R^{dL}$ by a given dataset. For every prefix length $\ell\in\{1,\ldots,L\}$, the induced KDE function on the length-$\ell$ prefixes of $X$ is given by
\[
 KDE_X^{[:\ell]}(y) := \frac{1}{|X|}\sum_{x\in X}k^{[:\ell]}(x,y)
\]
for every $y\in\R^{d\ell}$. In order to sequentially sample a length-$L$ sequence, we need to evaluate these KDE functions sequentially for each value of $\ell$. 

We start by recalling the formal DP KDE result from \cite{wagner2023fast}, adapted to our notation.
\begin{lemma}[Theorem 1.1 from \cite{wagner2023fast}]\label{thm:dpkde}
 Let $\varepsilon>0$ and $\alpha\in(0,1)$ be such that $|X|\geq O(1/(\varepsilon\alpha^2))$. Then, one can construct in time $O(|X|dL/\alpha^2)$ an $\varepsilon$-DP data structure for $KDE^{[:L]}_X$, such that for every $y\in\R^{dL}$, the value of $KDE^{[:L]}_X(y)$ can be reported in time $O(d/\alpha^2)$ with probability $0.99$ up to additive error at most $\alpha$. 
\end{lemma}

We extend this result to show that the same $\varepsilon$-DP data structure can in fact be used for all the KDE functions $\{KDE_X^{[:\ell]}:\ell=1,\ldots,L\}$ simultaneously, by padding any query $y\in\R^{d\ell}$ into $\bar{y}\in\R^{dL}$ by placing zeros the missing dimensions, and query the DP KDE data structure for $KDE^{[:L]}_X(\bar y)$. This what enables our Logarithmic DP KDE Ensemble to be frugal in the number of DP KDE data structures. 

\begin{theorem}\label{thm:ensemblekde}
 In notation of \Cref{thm:dpkde}, suppose that every vector in $X$ is the concatenation of $L$ $d$-dimensional vectors of squared norm $u$. Then, for every $\ell\in\{1,\ldots,L\}$, and for every $y\in\R^{d\ell}$ which is the concatenation of $\ell$ $d$-dimensional vectors of squared norm $u$, the $\varepsilon$-DP data structure for $KDE^{[:L]}_X$ from \Cref{thm:dpkde} can be used to report $KDE^{[:\ell]}_X(y)$ up to an additive error of $\alpha\cdot e^{2u(L-\ell)}$ with probability $0.99$, by querying $KDE^{[:L]}_X(\bar y)$, where $\bar y\in\R^{dL}$ is the result of zero-padding $y$. 
\end{theorem}
\begin{proof}
We recap the DP KDE construction of \cite{wagner2023fast}, which is based on Random Fourier Features \cite{rahimi2007random}. 
To construct the DP KDE data structure,
let $I=O(1/\alpha^2)$. For every $i=1,\ldots I$, 
sample $\omega_i\sim N(0,I_{dL})$ (a $dL$-dimensional vector of independent standard Gaussians), and a uniformly random $\beta_i\in[0,2\pi)$. Let $f_i:\R^{dL}\rightarrow[-\sqrt{2},\sqrt{2}]$ be defined as $f_i(z)=\sqrt{2}\cos(\sqrt{2}\omega_i^Tz+\beta_i)$, and let $g_i$ denote the same function as $f_i$ (the reason for this duplicate notation would become clear later, where our extention of this analysis would use different $f$ and $g$). Let $\tilde F_i(X)=\frac{1}{|X|}\sum_{x\in X}f_i(x)+\Lambda_i$, where $\Lambda_i\sim\mathrm{Laplace}(2\sqrt{2}I/\varepsilon)$. 
By the standard DP Laplace mechanism, the collection $\{\tilde F_i(X):i=1,\ldots,I\}$ is $\varepsilon$-DP and safe to release. 
Upon receiving a query $y\in\R^{dL}$, return $\frac{1}{I}\sum_{i=1}^I\tilde F_i(X)g_i(y)$ as the DP estimate for $KDE_{X}^{[:L]}(y)$. The analysis is based on the following lemma from \cite{wagner2023fast},
\begin{lemma}\label{lmm:dpkde}
If $\E[f_i(x)g_i(y)]=k^{[:L]}(x,y)$ for every $x,y\in\R^{dL}$ (where the randomness is over the sample of $\omega_i$ and $\beta_i$), the DP KDE estimate is accurate under the conditions of \Cref{thm:dpkde}.
\end{lemma}
Since \cite{rahimi2007random} showed that $\E[f_i(x)g_i(y)]=k^{[:L]}(x,y)$ for the above defined $f_i$ and $g_i$, it follows from \cref{lmm:dpkde} that the DP KDE estimate for $KDE_X^{[:L]}(y)$ is accurate.

Now, for some $\ell\in\{1,\ldots,L\}$, consider a prefix query $y\in\R^{d\ell}$ and its zero-padded completion $\bar y\in\R^{dL}$. Note that in our notation from above, $\bar y^{[:\ell]}=y$ and $\bar y^{[\ell:]}=0$. 
Consider what happens when we query the DP KDE structure for $\bar y$. For every fixed $i\in\{1,\ldots,I\}$ and $x\in X$, we have
\begin{align*}
 \E[f_i(x)g_i(\bar y)] = k^{[:L]}(x,\bar y) &= \exp(-\norm{x-\bar y}_2^2) & \text{by \cite{rahimi2007random}} \\
 &= \exp(-(\norm{x^{[:\ell]}-\bar y^{[:\ell]}}_2^2 + \norm{x^{[\ell:]}-\bar y^{[\ell:]}}_2^2)) & \\
 &= \exp(-(\norm{x^{[:\ell]}-y}_2^2 + \norm{x^{[\ell:]}}_2^2)) & \text{$\bar y^{[:\ell]}=y$ and $\bar y^{[\ell:]}=0$}\\
 &= \exp(-(\norm{x^{[:\ell]}-y}_2^2 + u(L-\ell))) & \norm{x^{[\ell:]}}_2^2=u(L-\ell) \\
 &= \exp(-\norm{x^{[:\ell]}-y}_2^2)\cdot \exp(-u(L-\ell)) & \\
 &= k^{[:\ell]}(x,y)\cdot \exp(-u(L-\ell)) . &
\end{align*}
Thus, if we now define $g_i^{(\ell)}:\R^{d\ell}\rightarrow\R$ as $g_i^{(\ell)}(y) = e^{u(L-\ell)} \cdot g_i(\bar y)$,
we get
\[ \E[f_i(x)g_i^{(\ell)}(\bar y)] = k^{[:\ell]}(x,y) . \]
In the terminology of \cite{wagner2023fast}, this means that $\{f_i,g_i^{(\ell)}\}$ is a $(1,e^{u(L-\ell)},1)$-LSQ family for the asymmetric kernel $k^{[:\ell]}$. \Cref{lmm:dpkde} now implies that the private KDE estimate we report is accurate up to an additive error of $\alpha\cdot e^{2u(L-\ell)}$, as the theorem claims (the $e^2{u(L-\ell)}$ blowup is due to the $e^{u(L-\ell)}$ term in the middle LSQ parameter; see Lemma 2.5 in \cite{wagner2023fast}). 
This holds for every $\ell=1,\ldots,L$. Furthermore, the left-function $f_i^{\ell}$ in the LSQ family is the same ($f_i$) for every $\ell$, and the privatized values released by the DP KDE mechanism, $\{\tilde F_i(X):i=1,\ldots,I\}$, depend only the left function. Therefore, the same DP KDE mechanism can be used simultaneously for all the KDE functions $KDE^{[:\ell]}_X$, i.e., for all subsequence lengths $\ell=1,\ldots,L$.
\end{proof}

\noindent\emph{Bandwidth selection.}
We now explain how to deal with the error blowup $e^{2u(L-\ell)}$ incurred in \Cref{thm:ensemblekde}. 
This has to do with the notion of bandwidth in KDE. 
Normally, the Gaussian KDE is defined with a bandwidth parameter $\sigma>0$, as $KDE_X(y)=\frac{1}{|X|}\sum_{x\in X}\exp(-\norm{y-x}_2^2/\sigma^2)$. 
Note that selecting a bandwidth $\sigma$ is equivalent to multiplying all vectors $X$ and $y$ by the scalar $\sigma$. 
Thus, the purpose of the bandwidth is to offset the length of the input vectors, and ensure that the KDE function around any point in $X$ decays not too fast and not too slowly, rendering its density scores meaningful. 

In the case of a pre-trained embedding model that returns unit-length embeddings (like many commonly used ones do), one simply sets the bandwidth to $1$. In our case, when we deal with the concatenation of $\ell$ unit-length embeddings, the Euclidean norm of the concatenation is $\sqrt{\ell}$, and thus we wish to set the bandwidth to $\sigma=1/\sqrt\ell$. This is equivalent to setting $u=1/\ell$ in the notation of \Cref{thm:ensemblekde} (recall that $u$ is the squared Euclidean norm of every $d$-dimensional block in the concatenation).

Thus, to get the desired bandwidth for prefix length $\ell$ (i.e., to ger meaningful density scores from $KDE^{[:\ell]}_X$), ideally we would normalize the given $d$-dimensional embeddings to length set $u=1/\ell$. 
However, if we use a single DP KDE data structure, we can only choose one setting of $u$ for all values of $\ell$. This is the motivation for using logarithmically many DP KDE data structures, leading to the logarithmic ensemble from \Cref{sec:method_main}. 
To explain this, suppose we build a DP KDE data structure for $KDE^{[:L]}_X$, and set $u=2/L$. Let $\ell$ be such that $\ell\geq L/2$. 
On the one hand, the ideal setting of $u$ for $\ell$ would have been $1/\ell$, while the actual setting $2/L$ is in $[1/\ell, 2/\ell]$. Thus, the bandwidth is off by only a constant, which suffices for the KDE density scores to still be meaningful. 
On the other hand, the error blowup $e^{2u(L-\ell)}$ is now at most $e^2$, i.e., \Cref{thm:ensemblekde} incurs only a constant blowup in the error. Hence, if we pad no more than half of the blocks of a KDE query with zeros, we get both approximately the desired bandwidth and the desired error blowup. To achieve this for every $\ell$, we build $O(\log L)$ many DP KDE data structures, such that for every $\ell$ we can query one that requires no more than doubling the dimension of the query by zero padding, as described in detail earlier in this section. This is the motivation for the Logarithmic Ensemble, and this is how it achieves private and accurate DP KDE estimates for all prefix lengths with a logarithmic number of KDE data structures.

\section{Additional Experiments}\label{sec:experiments_appendix}

\subsection{Experimental Details}
\subsubsection{Medical Task Details}\label{app:mimic}
\noindent\textit{Groundtruth labels.}
As the groundtruth labels for the binary classification task on MIMIC, we consider as the positive class all medical records whose billing code includes an ICD code associated with heart failure. These are the codes 428.x (for any x) in ICD-9,\footnote{\url{http://www.icd9data.com/2012/Volume1/390-459/420-429/428/default.htm}}
and the codes I50.x in ICD-10.\footnote{\url{https://www.icd10data.com/ICD10CM/Codes/I00-I99/I30-I5A/I50-}}

\noindent\textit{Clinical Vocabulary.}
Physicians express the same concept in different ways in clinical notes. For example, the term hypertension may be expressed as ‘high blood pressure’ or ‘elevated blood pressure’. Abbreviations are also common such as ‘htn’ or ‘hbp’. Ideally, we would treat these terms as the concept. Fortunately, the UMLS has a meta-thesaurus with 3.2M English medical concepts. Each concept, known as a CUI, has a set of keyphrases associated with it.
To limit the vocabulary size, we chose a subset of common concepts based on SemMedDB \cite{kilicoglu2012semmeddb}, a relational graph over these CUIs based on the PubMed database. The graph contains relationships such as ``‘Lisinopril’ is a medication for ‘high blood pressure’''. We chose concepts that were supported by at least 10 PubMed papers. This resulted in 386,725 CUIs that formed the final public vocabulary for our MIMIC experiments.

\subsubsection{Domain Adaptation Details}\label{sec:deep_coral_details}
For the domain adaptation step in our implementation of \alg, we use the Deep CORAL method of \cite{sun2016deep}. We now describe it in more detail. 

Recall the setting: Hospital A has generated a privacy-preserving synthetic corpus $D_A$ of labeled texts, which is safe for release, and has sent it to Hospital B. Hospital B now intends to train a downstream ML model $\mathcal{M}$ on $D_A$ for some specific task (say, classification), and then use it for inference on its own unlabeled text corpus, $D_B$. However, since the texts in $D_A$ might have different characteristics than $D_B$ (e.g., a different formats of medical records), domain adaptation is needed in order to use $D_A$ as the training set for $\mathcal{M}$ and yet get good inference accuracy on $D_B$. 

Let $\ell_{\mathcal M}$ be the loss used to train $\mathcal{M}$ (say, classification loss). Deep CORAL modifies $\ell_{\mathcal M}$ into a new loss $\ell_{\mathcal M}'$ by adding an additive domain adaptation loss term, defined as $1/(4d^2\norm{C_A-C_B}_F^2)$, where $C_A$ and $C_B$ are the $d\times d$ covariance matrices constructed from $D_A$ and $D_B$ respectively. $\mathcal{M}$ is trained with the modified loss $\ell_{\mathcal M}'$ instead of the original loss $\ell_{\mathcal M}$. See \cite{sun2016deep} for further details.

This method requires Hospital B to allocate a set $D_B$ of its medical records for training  $\mathcal{M}$. The size we set for $D_B$ is half the size of $D_A$ (thus, $|D_B|$ contains 7k texts for DBPedia-14, and 3k texts for MIMIC-IV). The test set on which we evaluate the trained model $\mathcal{M}$ for the accuracy results we report is kept separate from the set $D_B$ used for training $\mathcal{M}$ with domain adaptation, to ensure $\mathcal{M}$ has no access to any test records during training.

\subsubsection{AugPE Details}\label{sec:augpe_details}
In all invocations of AugPE, we use 10 epochs, similarly to \cite{api2}. Thus, AugPE uses 10 sequential prompts to the LLM per synthetic text generated. To vary the prompt budget of AugPE (or AugPE-$\widetilde{\mathcal{V}}$), we generate as many synthetics texts as possible under the budget in 10 epochs (thus, with a prompt budget $B$, we generate $B/10$ synthetic documents with AugPE). 
We have also run experiments varying AugPE's prompt budget by using less epochs (which allows generating more documents with the same prompt budget; namely, with a prompt budget $B$ and $M<10$ epochs, one can generate $B/M$ synthetic texts). However, this led to degraded performance compared to fixing the number of epochs to 10 as above---which stands to reason, since evolution over epochs is the key idea in AugPE---and thus we report results with the better variant of the algorithm.

%Note that no private clinical records were used to select the vocabulary – just published PubMed papers. But the density associated with each CUI was based on a private kernel density estimate.

%Public data was used to generate a clinical vocabulary. However,

\subsection{Additional Experimental Results}

%\subsection{Ablations}
%\subsection{Post-Processing and Domain Adaptation}
%\noindent\textbf{Domain adaptation.}
\subsubsection{Ablation: Domain Adaptation}
\Cref{tbl:domainadaptation} shows that when domain adaptation (see \Cref{sec:method_main}) is removed from \alg, downstream accuracy is significantly degraded.

\begin{table*}
\caption{Results with and without domain adaption (abbrev.~DA).} \label{tbl:domainadaptation}
\centering
\begin{tabular}{c c c c c}
\toprule
 & \multicolumn{2}{c}{MIMIC} & \multicolumn{2}{c}{DBPedia-14} \\
 $\varepsilon_{\mathrm{total}}$ & with DA & w/o DA & with DA & w/o DA \\
\midrule
\rule{0pt}{3ex}
6 & 71.6\% & 58.7\% & 77.9\% & 72.8\% \\
10 & 71.7\% & 58.4\% & 80.4\% & 77.1\% \\
11 & 72.2\% & 55.0\% & 82.4\% & 77.0\% \\
15 & 72.2\% & 60.3\% & 85.1\% & 81.2\% \\
\bottomrule
\end{tabular}
\end{table*}

\subsubsection{Ablation: No LLM}\label{sec:nollm}
%\vspace{-8pt}
%\noindent\textbf{No LLM.}
We consider the ablation of eliminating the LLM from \alg. 
This pertains to the question of what is the role of the LLM in the synthetic text generation pipeline. Clearly, it plays a role in \emph{form}, by incorporating the sequences of isolated phrases into natural language. However, potentially, it could also play a role in \emph{function}: either by augmenting and enriching the synthetic texts with knowledge about the world from its pre-training data, which could assist downstream tasks, or conversely, by adding ``noise'' that would only interfere. 
Note that is it generally not at all clear (and may depend on the specific context) whether it is desirable or not for the LLM to affect the output texts beyond their form. For example, in medical contexts, generic information about the general population from the pre-training data might obscure predictive attributes in datasets that target specific sub-populations.

To test this, we run downstream classification directly on the generated phrases sequences, instead of using them to seed an LLM prompt. Each test record is also turned into a (non private) sequence of phrases from the vocabulary, to be compatible in form with the raw sequence training set. 

The results are in \Cref{tbl:nollm}. They show that on DBPedia-14, the accuracy remains similar with or without the LLM. On MIMIC, the classification accuracy of the sequences is in fact better. However, this gap might not be due to the LLM, since the same gap is observed between classifying the original texts versus classifying the sequences extracted from them (the final lines in \cref{tbl:mainres,tbl:nollm},
respectively). 
The reason is that medical records are long and packed with a lot of free text, and classification becomes easier once key medical terms from a standardized glossary (UMLS) are extracted from the text. To summarize, we do not observe evidence that the LLM substantially affects downstream accuracy either positively or negatively, and the role it plays seems to be primarily in the form of the output texts.

%suggesting that it only plays a role in form, and does not substantially affect function either positively or negatively. 
%Note that in the context of DBPedia-14, this may suggest that the LLM indeed does not retrieve the Wikipedia articles corresponding to the dataset records from its pre-training set.

\begin{table*}
\caption{Classification accuracy on the generated phrase sequences, without running them through the LLM.} \label{tbl:nollm}
\centering\begin{tabular}{c c c c c}
\toprule
$\varepsilon_{\mathrm{total}}$ & $\varepsilon_{\mathrm{voc}}$ & $\varepsilon_{\mathrm{kde}}$ & MIMIC & DBPedia-14 \\
\midrule
\rule{0pt}{3ex}
6 & 1 & 5 & 75.1\% & 71.7\% \\
10 & 5 & 5 & 74.9\% & 81.5\% \\
11 & 1 & 10 & 75.5\% & 80.6\% \\
15 & 5 & 10 & 75.9\% & 84.2\% \\
\midrule
\multicolumn{3}{c}{Original training set} & 80.0\% & 85.2\% \\
\bottomrule
\end{tabular}
\end{table*}

\subsubsection{Ablation: Only LLM}
%\vspace{-10pt}
%\noindent\textbf{Only LLM.}
The flipside of the previous ablation is cutting out everything \emph{but} the LLM from the pipeline, not using private data at all. Perhaps the LLM has enough knowledge from its pre-training set to generate the requisite texts for the downstream task, without needing any private data?

To test this, we prompted the LLM without seeding to generate documents for each class of each dataset, and ran downstream classification on the generated texts. Note that here, unlike our \alg\ experiments, we did include the intended class in the prompt (say, ``generate a medical record of a patient with a heart condition''), since without phrase seeding, there is no other input for the prompt to rely on. 

The resulting texts were extremely non-diverse (e.g., all ``mean of transportation'' articles were about Honda Civic). 
The downstream classification accuracy was 57\% on MIMIC and 60\% on DBPedia-14. While non-trivial, these accuracies fail to match those based on \alg\ seeded prompts. 
This result aligns with and corroborates a similar finding in \cite{eldan2023tinystories}, discussed in \Cref{sec:intro_method}.
%
%We remark that lack of diversity in the LLM outputs for the unseeded prompts, and their poor performance on downstream, corroborate and essentially replicate a similar finding from \cite{eldan2023tinystories}, as discussed in \Cref{sec:intro_method}. It also corroborates the motivation of \cite{cohen2023hot} to rigorously study the interplay between differential privacy and generated output diversity. 

\subsubsection{Sequence Generation Method}

We compare independent versus iterative sequence generation (see \Cref{sec:method_main}). The results vary substantially on the two datasets: while on MIMIC independent generation works much better, the situation is reversed on DBPedia-14. There are a few possible causes: one is that MIMIC medical records are long and detailed, and cover many medical aspects not directly related to each other nor to the classification task, while DBPedia-14 entries are focused and concise. Another is that the MIMIC dataset is much smaller, perhaps too small to display strong correlations between longer tuples of phrases (particularly in the presence of DP noise). The results imply that the choice between the method is dataset-dependent, and is best done via validation. Note that validation does not require costly interaction with the LLM, and can be done cheaply directly on the generated sequences, as done in the experiment in \Cref{tbl:nollm}.

\begin{table}
\caption{Independent vs.~iterative phrase generation.} \label{tbl:iidvsiter}
  \centering
\begin{tabular}{c c c c c}
\toprule
 & \multicolumn{2}{c}{MIMIC} & \multicolumn{2}{c}{DBPedia-14} \\
 $\varepsilon_{\mathrm{total}}$ & ind. & iter. & ind. & iter. \\
\midrule
\rule{0pt}{3ex}
6 & 71.6\% & 51.3\% & 30.2\% & 77.9\% \\
10 & 71.7\% & 61.5\% & 26.3\% & 80.4\% \\
11 & 72.2\% & 55.3\% & 29.0\% & 82.4\% \\
15 & 72.2\% & 55.5\% & 27.1\% & 85.1\% \\
\bottomrule
\end{tabular}
\end{table}

\subsubsection{Few-Shot Prompting}\label{app:fewshot}

As mentioned in \Cref{sec:method_main}, the privatized documents generated are intended to be used by a client for downstream machine learning tasks on their own set of documents. The synthetic documents that are produced by the LLM might look different compared to the test set. While one way to overcome this is through domain adaptation as described in \Cref{sec:deep_coral_details}, another approach is through few-shot prompting. 
In this approach, the client provides a few examples from their dataset in the desired format. These examples are then included in the prompt along with the synthetic keyphrases given to the LLM. The idea here is to nudge the LLM to generate documents that are aligned with the client's test set. 

We explore this post-processing method on the MIMIC and DBPedia-14 text classification tasks. We provide six example texts (3 from each class) from the dataset along with their corresponding keyphrases in the prompt. We ensure that these few-shot reports are not in the test-set used in our experiments. We compute the downstream task performance of a classifier trained on the synthetic data with few-shot prompting. The downstream results using the different combinations of post-processing methods are shown in \Cref{tbl:fs_results_mimic} and \Cref{tbl:fs_results_dbpedia}. We find that few-shot prompting along with domain adaptation results in the best downstream performance. 

To gauge the effect of few-shot prompting qualitatively, we visualize the t-SNE plot of the BioBERT embeddings of the real MIMIC data, the synthetic data (DP-KPS outputs w/o DA) and the synthetic data with few-shot prompting in \Cref{fig:tsne_biobert}. Our intuition for using few-shot prompting was to make the styles of the real and synthetic data similar. We see in \Cref{fig:tsne_biobert} that few-shot prompting brings the embeddings closer to the real data. 
 
\begin{table*}
\caption{Accuracy of a downstream classifier trained on DP-KPS outputs on MIMIC with different combinations of post-processing methods. These results are for the setting where $\varepsilon_{\mathrm{voc}}=1$ and $\varepsilon_{\mathrm{kde}} = 5$ } \label{tbl:fs_results_mimic}
\centering
\begin{tabular}{c c c}
\toprule
Domain adaptation & Few-shot prompting & Accuracy \\
\midrule
\rule{0pt}{3ex}
{\text{\sffamily X}} & {\text{\sffamily X}} & 58.7\% \\
\checkmark & {\text{\sffamily X}} & 71.6\% \\
{\text{\sffamily X}} & \checkmark & 71.0\% \\
\checkmark & \checkmark & 72.37\% \\
\bottomrule
\end{tabular}
\end{table*}

\begin{table*}
\caption{Accuracy of a downstream classifier trained on DP-KPS outputs on DBPedia with different combinations of post-processing methods. These results are for the setting where $\varepsilon_{\mathrm{voc}}=1$ and $\varepsilon_{\mathrm{kde}} = 10$ } \label{tbl:fs_results_dbpedia}
\centering
\begin{tabular}{c c c}
\toprule
Domain adaptation & Few-shot prompting & Accuracy \\
\midrule
\rule{0pt}{3ex}
{\text{\sffamily X}} & {\text{\sffamily X}} & 77.0\% \\
\checkmark & {\text{\sffamily X}} & 82.4\% \\
{\text{\sffamily X}} & \checkmark & 79.6\% \\
\checkmark & \checkmark & 83.3\% \\
\bottomrule
\end{tabular}
\end{table*}

\begin{figure}
 \centering
 \includegraphics[scale=0.45]{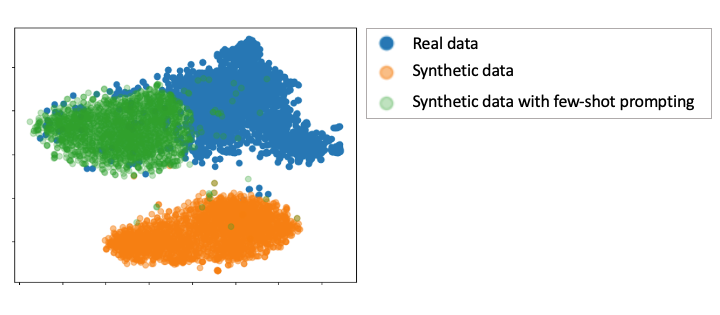}
 \caption{t-SNE plot of the BioBERT embeddings of real MIMIC data, synthetic data (DP-KPS output w/o DA) and synthetic data with few-show prompting.}
 \label{fig:tsne_biobert}
\end{figure}

\subsection{Examples of Generated keyphrase Sequences and LLM Outputs}
\Cref{tbl:example2,tbl:example1full} complete \Cref{tbl:example1} by showing positive and negative examples for all classes of the DBPedia-14 dataset. 

\begin{table*}
\caption{\textcolor{goodgreen}{Good} and \textcolor{badred}{bad} examples for classes 1--7 in DBPedia-14, with $\varepsilon_{\mathrm{voc}}=1$ and $\varepsilon_{\mathrm{kde}}=10$.} \label{tbl:example1full}
{\renewcommand{\arraystretch}{2}
\begin{centering}
\setlength\extrarowheight{-3pt}
\small
\begin{adjustwidth}{-0.5in}{-0.1in}
\begin{tabular}{p{0.08\linewidth}p{0.92\linewidth}}
\toprule
 Class & \alg\ results: Keyphrase sequence and corresponding LLM output \\
\midrule
Company &

['plants', 'company', 'distribution', 'operating', 'social', 'ships', 'andrew', 'based', 'michigan', 'york']: \newline 
\textcolor{goodgreen}{Plants Distribution Company is a Michigan-based company operating ships to distribute plants and other goods. Founded by Andrew York in New York, the small social enterprise distributes products across the region.} \\
&
['director', 'institution', 'royal', 'schools', 'municipality', 'club', 'journal', 'johnson', 'office', 'oil']: \newline 
\textcolor{badred}{John Johnson was appointed director of the Royal Schools institution in the municipality in 2020. He previously worked in the office of the Oil Club journal.} \\
\midrule

Educational institution
&
['1908', 'school', 'school', 'college', 'institute', 'school', 'germany', 'university', 'college', 'manufacturing']: \newline 
\textcolor{goodgreen}{The College Institute was founded in Germany in 1908 as a manufacturing school. Originally a vocational college, it later became a university specializing in engineering and technology. The school focuses on hands-on training in manufacturing techniques.} \\
&
['scott', 'school', 'brazilian', 'guard', 'sciences', '1996', 'high', 'lee', 'mystery', 'england']: \newline 
\textcolor{badred}{Scott Lee was born in England in 1996. He attended a Brazilian jiu-jitsu school in high school where he learned grappling and the martial art of Brazilian jiu-jitsu. The mystery sciences behind the techniques fascinated Scott during his time there.}\\
\midrule

Artist &

['species', 'february', '1944', 'european', '1929', 'music', 'musician', 'duo', 'frank', 'light']: \newline 
\textcolor{goodgreen}{The musical duo Frank \& Light was formed in February 1944 in Europe by Frank Smith and Light Williams. The duo, known for their unique blend of jazz and classical music, released their debut album "Species of Music" in 1929 to critical acclaim. Though popular in the late 1940s, they disbanded as a musical act in the early 1950s.} \\

& ['england', 'june', '1972', '2003', 'english', 'animated', 'nonfiction', 'publisher', 'utah', 'november']: \newline 
\textcolor{badred}{In November 2003, an English animated nonfiction publisher based in Utah released a book about events that took place in England in June 1972. The book was well-received upon its November 2003 publication. }\\

\midrule
 Athlete &

['early', 'born', '1966', '1970', 'class', 'professional', 'playing', 'plays', 'lee', 'brown']: \newline 
\textcolor{goodgreen}{Lee Brown was born in 1966 and grew up playing sports from an early age. He was a standout athlete in high school and went on to play professionally in the late 1980s and early 1990s before retiring in his late 20s. After retiring from playing, Brown began coaching youth sports and teaching physical education classes. } \\

&

['movie', 'october', '1988', 'irish', 'british', 'footballer', 'league', 'attack', '1951', 'formed']: \newline 
\textcolor{badred}{The 1988 Irish movie October follows a fictional British footballer in the Irish league who was attacked in 1951. The footballer then forms a new team in October 1988.} \\

\midrule
Office holder &

['china', 'march', '1948', '1999', 'politician', 'canadian', 'national', 'leader', 'anthony', 'district']: \newline 
\textcolor{goodgreen}{In March 1948, Anthony, a Canadian politician and national leader, was born in a district of China. In 1999, he became the leader of his district.} \\
&
['chapel', 'january', 'november', '1947', 'kentucky', 'senior', 'year', 'journalist', 'democratic', 'hop']: \newline 
\textcolor{badred}{The Little Chapel was founded in January 1947 in Kentucky by a senior journalist from the Democratic Party who hoped it would last. The Little Chapel opened its doors in November 1947 but closed within a year due to lack of funding. The Little Chapel story was covered by the local Kentucky paper in 1947. } \\

\midrule
 Mean of transportation &

['malayalam', 'mexican', 'ocean', 'car', 'submarine', 'finnish', 'frigate', 'called', 'sea', 'union']: \newline 
\textcolor{goodgreen}{The Malayalam is a Mexican frigate that operates in the Pacific Ocean. The sea vessel, called the Union, carries specially designed submarines and cars for traversing the ocean floor. The Finnish captain aims to better understand deep sea life. }\\
&
['girls', 'ship', 'oil', 'automobile', 'navy', 'massachusetts', 'ships', 'islands', 'military', 'wales']: \newline 
\textcolor{badred}{The girls were on a ship transporting oil to Wales when it was attacked by the Massachusetts navy near some islands. The military ships fired on the automobile transport ship until it sank.}
\\

\midrule
 Building
&
['career', 'museum', 'tower', 'representing', 'headquarters', 'residence', 'japanese', 'elected', 'music', 'forest']: \newline 
\textcolor{goodgreen}{The Japanese music museum's tower, located in a forest, was the former residence and career headquarters of an elected official representing Japanese culture before becoming a museum. The tower in the forest now serves as the music museum's main headquarters. }\\
&
['mexico', 'city', 'headquartered', 'west', 'chapel', 'news', 'game', 'alpine', 'northwest', 'comprehensive']: \newline 
\textcolor{badred}{The city of Alpine is located in northwest Mexico and is the headquarters of West Chapel News, a company that publishes a comprehensive game and news website. }\\

\bottomrule
\end{tabular}
\end{adjustwidth}
\end{centering}}
\end{table*}

\begin{table*}
\caption{\textcolor{goodgreen}{Good} and \textcolor{badred}{bad} examples for classes 8--14 in DBPedia-14, with $\varepsilon_{\mathrm{voc}}=1$ and $\varepsilon_{\mathrm{kde}}=10$.} \label{tbl:example2}
{\renewcommand{\arraystretch}{2}
\begin{centering}
\setlength\extrarowheight{-3pt}
\small
\begin{adjustwidth}{-0.5in}{-0.1in}
\begin{tabular}{p{0.08\linewidth}p{0.92\linewidth}}
\toprule
 Class & \alg\ result: Keyphrase sequence and corresponding LLM output \\
\midrule
Natural place &

['1955', 'mountain', 'lakes', 'province', 'carolina', 'north', 'north', 'jersey', 'maryland', 'region']: \newline 
\textcolor{goodgreen}{In 1955, the mountain lakes region spanning parts of North Carolina, North Jersey, and Maryland was designated as a new province. This area contains beautiful mountain lakes and forests, and the region attracts many tourists.} \\
&
['presbyterian', 'mountain', 'swimmer', 'east', 'forest', 'rapper', 'lakes', 'international', 'arts', 'india']: \newline 
\textcolor{badred}{The Presbyterian Mountain Swimmer is an international rapper and performance artist known for swimming across lakes and forests in East India. His arts performances often address social issues in local mountain communities near the lakes and forests where he swims.} \\
\midrule

Village
&
['2009', 'village', 'atlantic', 'spring', 'southern', 'folk', '1986', 'area', 'female', 'big']: \newline 
\textcolor{goodgreen}{The village of Big Spring is located in the southern area of the Atlantic coast. Founded in 1986, it has become a popular tourist destination known for its folk music scene. In 2009, a local female musician released a popular album that brought more visitors to the village.} \\
&
['1988', 'district', 'pass', 'general', 'manufacturing', 'power', '1953', 'landing', 'leader', 'cemetery']: \newline 
\textcolor{badred}{In 1953, the district leader passed a general manufacturing power bill to allow more factories to be built in the area. However, residents complained about pollution from the new factories near the local cemetery. So in 1988, the district passed new restrictions on factory emissions to address residents' concerns over poor air quality.}\\
\midrule

Animal &

['martin', 'species', 'african', 'patrol', 'creek', 'portuguese', 'asia', 'ukrainian', 'plants', '1927']: \newline 
\textcolor{goodgreen}{The African creek martin (Riparia africanus) is a species of small insectivorous bird in the swallow family. First described in 1927 by a Portuguese naturalist, the creek martin nests in burrows dug into the banks of creeks and rivers across sub-Saharan Africa. Though once abundant, habitat loss has caused populations to decline, leading conservation groups like the Ukrainian Wildlife Patrol to list the species as near threatened.} \\

& ['feet', 'species', 'argentine', 'public', 'metal', 'russian', 'bay', 'lakes', 'governor', 'wild']: \newline 
\textcolor{badred}{The public metal sculptures in Lakes Bay have become a popular tourist attraction. Featuring over 50 feet tall structures depicting various species, the sculptures were a gift from the Russian governor to the people of Argentine. However, some environmental groups have expressed concern that the wild installations could negatively impact the local wildlife.}\\

\midrule
 Plant &

['plant', 'flowering', 'plant', 'family', 'common', 'danish', 'northwestern', 'india', 'norway', 'church']: \newline 
\textcolor{goodgreen}{The flowering plant family Commonaceae is commonly found in northwestern India, Norway, and other Northern European countries. The plant is known for its use in traditional medicine and as decoration in churches across Denmark and northern India.} \\

&

['june', 'plant', 'plant', 'fish', 'czech', 'ohio', 'multinational', 'western', 'east', 'representing']: \newline 
\textcolor{badred}{In June 2023, a multinational corporation headquartered in Ohio announced plans to build a new manufacturing plant in the Czech Republic, representing the company's expansion into Eastern Europe. The plant, to be located west of Prague, will produce parts used in automotive and aerospace industries. Environmental activists raised concerns about the plant's potential impact on local fish populations in nearby rivers.} \\

\midrule
Album &

['pass', 'album', 'album', 'released', 'duo', 'release', 'album', 'early', '1999', '2004']: \newline 
\textcolor{goodgreen}{The album Pass was released in early 1999 by the duo. This was their first album release, coming five years after they formed in 2004. The album contains 10 songs in the pop genre.} \\
&
['operated', 'album', 'album', 'distributed', 'hits', 'title', 'music', 'michael', 'shrub', 'lee']: \newline 
\textcolor{badred}{Michael Shrub is an American musician who released the album "Evergreen Hits" in 2022, which was distributed by Lee Records. The album contains several popular songs that showcase Shrub's smooth vocals and guitar playing.} \\

\midrule
Film &

['published', '1966', 'film', 'norwegian', 'director', 'stage', 'food', 'light', 'sports', 'production']: \newline 
\textcolor{goodgreen}{The 1966 Norwegian film "Bright Lights" was directed by famous stage director Lars Berg and produced by Nordic Productions. The sports comedy film about competitive eating features innovative lighting and sumptuous food.}\\
&
['named', '1966', 'film', 'swedish', 'actress', 'operates', 'animated', 'lit', '1986', 'music']: \newline 
\textcolor{badred}{The Swedish actress named Greta was born in 1966 and first operated an animated film lit in 1986. She later began making music.}
\\

\midrule
 Written work

&
['comic', 'english', 'nonfiction', 'newspaper', '1965', 'institute', '1992', 'documentary', 'jack', '1990']: \newline 
\textcolor{goodgreen}{The English comic strip Jack debuted in 1965 in a nonfiction newspaper published by the Institute. The comic ran until 1992 and was known for its documentary-style depiction of everyday life. In 1990, the Institute published a documentary about the 25-year history of the Jack comic strip.}\\
&
['songwriter', 'language', 'period', 'greater', 'louisiana', 'privately', 'literary', 'france', 'multinational', 'ireland']: \newline 
\textcolor{badred}{John Smith (songwriter) was a privately educated language teacher from Ireland who moved to France during the greater Louisiana Purchase period. He wrote literary works while working for a multinational corporation before returning to Ireland. }\\

\bottomrule
\end{tabular}
\end{adjustwidth}
\end{centering}}
\end{table*}

\newpage
\section{Impact Statement}
Our work suggests the potential use of LLM generated texts as synthetic data in machine learning tasks, as a manner of preserving privacy. Along with the many potential benefits of LLMs, using them is prone to known risks like hallucinations, misinformation, and introduction of biases from the pre-training set into the generated data. 
AI generated data is automated and probabilistic, and may often be inaccurate or inappropriate. 
These risks are present in virtually any use of generative AI models and one should always be mindful of them, but our view is that the specific use our work proposes, of using them to protect privacy, does not encompass substantial risks beyond those always present in generative model usage. 

A best practice for prospective users of our method would be to validate that the generated data faithfully represents the essential attributes of the original data (for example by comparing it to the original data in sample downstream learning tasks) prior to releasing it, and to transparently convey and emphasize with its release that the privatized data, while seeded with real private data, had been automatically generated by an AI generative model.

On the client side, a responsible use of AI generated data (by our method or any other one), particularly for any consequential decisions with impact on humans, should implement appropriate human oversight, testing, and other use case-specific safeguards to mitigate the associated risks.

\end{document}